%% file: paper.tex
\documentclass[]{fairmeta}

\usepackage{amsmath}
\usepackage{amssymb}
\usepackage{amsthm}
\usepackage{tikz}
\usetikzlibrary{trees}
\usepackage{float}
\tcbuselibrary{listings}

\input{math_commands.tex}

\newtheorem{theorem}{Theorem}[section]

\DeclareMathOperator{\Prob}{\mathbb{P}}
\DeclareMathOperator{\lme}{\mathrm{logmeanexp}}

\newcommand{\Norm}{\mathcal{N}}

\definecolor{lightblue}{HTML}{DAE9FD}
\definecolor{violet}{HTML}{9482C4}
\newtcblisting{promptboxsmall}[2][]{%
    listing options={
        breaklines=true,
        breakautoindent=false,
        breakindent=0ex,
        basicstyle=\ttfamily\scriptsize,
        aboveskip=1pt,
        belowskip=1pt,
        escapechar=|,
    },
    width=\textwidth,
    listing only,
    autoparskip,
    colback=blue!2!white,
    colbacktitle=lightblue,
    coltitle=black,
    title={#2},
    fonttitle=\normalsize,
    #1
}
\newtcblisting{promptbox}[2][]{%
    listing options={
        breaklines=true,
        breakautoindent=false,
        breakindent=0ex,
        basicstyle=\ttfamily\footnotesize,
        aboveskip=1pt,
        belowskip=1pt,
    },
    width=\textwidth,
    listing only,
    autoparskip,
    colback=blue!2!white,
    title={#2},
    fonttitle=\normalsize,
    #1
}
\newtcblisting{promptboxpartone}[2][]{%
    listing options={
        breaklines=true,
        breakautoindent=false,
        breakindent=0ex,
        basicstyle=\ttfamily\scriptsize,
        aboveskip=1pt,
        belowskip=1pt,
        escapechar=|,
    },
    width=\textwidth,
    listing only,
    autoparskip,
    colback=blue!2!white,
    colbacklower=blue!2!white,
    colbacktitle=lightblue,
    coltitle=black,
    title={#2},
    fonttitle=\normalsize,
    breakable,
    #1
}
\newtcblisting{promptboxpartviolet}[2][]{%
    listing options={
        breaklines=true,
        breakautoindent=false,
        breakindent=0ex,
        basicstyle=\ttfamily\scriptsize,
        aboveskip=1pt,
        belowskip=1pt,
        escapechar=|,
    },
    width=\textwidth,
    listing only,
    autoparskip,
    colback=blue!2!white,
    colbacklower=blue!2!white,
    colbacktitle=violet,
    coltitle=white,
    title={#2},
    fonttitle=\normalsize,
    breakable,
    #1
}

\title{DecompRL: Solving Harder Problems by Learning Modular Code Generation}

\author[1,2*]{Juliette Decugis}
\author[1,3*]{Fabian Gloeckle}
\author[2]{Francis Bach}
\author[1]{Taco Cohen}
\author[1]{Gabriel Synnaeve}

\affiliation[1]{FAIR at Meta}
\affiliation[2]{Inria, \'Ecole Normale Sup\'erieure}
\affiliation[3]{CERMICS, \'Ecole des Ponts ParisTech}

\contribution[*]{Equal contribution}

\abstract{How can Large Language Models (LLMs) solve problems they currently cannot? Repeated sampling scales test-time compute but GPU cost grows linearly with attempts, while reinforcement learning (RL) with verifiable rewards improves single-attempt accuracy at the expense of sample diversity. Both strategies ultimately fail when the base policy has near-zero probability of producing a correct solution: no amount of sampling or gradient signal can overcome a search space that is simply too large. We take a different approach: rather than sampling harder, we make the task easier by decomposing problems into smaller, independently solvable sub-functions whose implementations can be recombined. Since off-the-shelf models are not trained for this modular generation, we introduce DecompRL, an RL algorithm that explicitly learns to decompose and implement hierarchical code structures. Recombining $k$ implementations of $n$ modules yields up to $k^{n}$ candidate solutions, shifting the bottleneck from GPU inference to cheap CPU evaluation and cutting GPU token cost by $\sim$50$\times$. On LiveCodeBench and CodeContests (Qwen~2.5~7B, Code World Model~32B), DecompRL outperforms standard and diversity-optimized RL baselines beyond $10^5$ tokens per problem, solving problems that standard generation cannot reach.}

\date{\today}
\correspondence{Juliette Decugis at \email{jdecugis@meta.com}}

\begin{document}

\maketitle

\section{Introduction}

\input{introduction}

\section{Method}
\input{method}

\section{Experimental setup}
\input{experiments}

\section{Results}
\input{results_good}

\section{Analysis}
\input{ablations}

\section{Limitations and Discussion}
\input{limits}

\section{Related works}
\input{related_works}

\section{Conclusion}

\input{discussion}

\newpage
\bibliographystyle{assets/plainnat}
\bibliography{references}

\clearpage
\newpage
\beginappendix

\input{appendix}

\end{document}

%% file: math_commands.tex
\def\eqref#1{equation~\ref{#1}}

\def\1{\bm{1}}

\DeclareMathAlphabet{\mathsfit}{\encodingdefault}{\sfdefault}{m}{sl}
\SetMathAlphabet{\mathsfit}{bold}{\encodingdefault}{\sfdefault}{bx}{n}

\newcommand{\E}{\mathbb{E}}

\newcommand{\R}{\mathbb{R}}

\DeclareMathOperator{\Cov}{Cov}
\DeclareMathOperator{\Var}{Var}

\DeclareMathOperator{\Tr}{Tr}

%% file: introduction.tex
\input{figures/figure1}
Large language models (LLMs) continue to improve in automatically verifiable domains like competitive programming. Currently, the dominant strategy for state-of-the-art performance relies on repeated sampling: generating numerous potential solutions to pass a verifier: the pass@k regime \citep{chen2021evaluating}. However, the major limitation of this approach lies in its significant GPU computational costs for scaling inference. Attempts to mitigate this through reinforcement learning post-training have primarily focused on improving single-sample accuracy (pass@1) at the cost of pass@k in high sampling regimes for $k >> 100$. This reduction in diversity undermines its effectiveness when scaling inference
at test time, because performance gains grow logarithmically for linear computational costs (Appendix~\ref{app:gumbel}).

Our approach targets tasks where solutions are hierarchically decomposable into independently implementable modules and where correctness is automatically verifiable. This structural assumption holds broadly: competitive programming problems naturally decompose into helper functions; formal proofs factor into lemmas \citep{lample2022hypertree}; hardware modules compose into chip designs \citep{liu2024openllm}; and scientific codebases separate data processing, modeling, and evaluation \citep{haller2024peano}. Whenever verification is cheap (unit tests, proof checkers, simulators) but generation is expensive (LLM inference), modular recombination offers a path to scale search without proportionally scaling GPU cost.

We propose DecompRL: a reinforcement learning framework that fundamentally shifts the scaling bottleneck from expensive GPU inference to cheap CPU evaluation. Inspired by modular inference strategies \citep{zelikman2023parsel}, DecompRL trains models to decompose challenges into a hierarchy of independently solvable sub-functions that can be re-combined to form a complete solution. For example, for a sorting task, instead of generating a single monolithic solution, the model identifies $N$ sub-problems such as partitioning, merging, and base-case sorting. By generating $K$ implementations for each of the $N$ modules in parallel, we can recombine them into up to $K^N$ unique candidate solutions. This approach creates a combinatorial explosion of potential programs for a fixed inference budget of only $K \times N$ tokens, allowing us to scale search on cheap CPU clusters without increasing the GPU footprint.

By framing hierarchical modular inference as an RL objective, DecompRL prioritizes exploration (Section \ref{sec:max_eval}), maximizes the utility of recombinations (Section \ref{sec:recomb}) and successfully discovers sparse rewards (Section \ref{sec:sota}). We do not position DecompRL as a replacement for standard RL post-training, which remains more sample-efficient when the starting policy is already strong. Rather, DecompRL is an exploration tool: it discovers correct solutions to problems that standard high-pass@$k$ sampling cannot reach, making it a natural complement for policy distillation pipelines and for overcoming the cold-start problem on hard tasks where the base policy has near-zero solve rates.

\input{figures/old_fig1}

%% file: figures/figure1.tex
\begin{figure}[!h]
    \centering
    \begin{minipage}[l]{\textwidth}
        \centering
        \includegraphics[width=\linewidth]{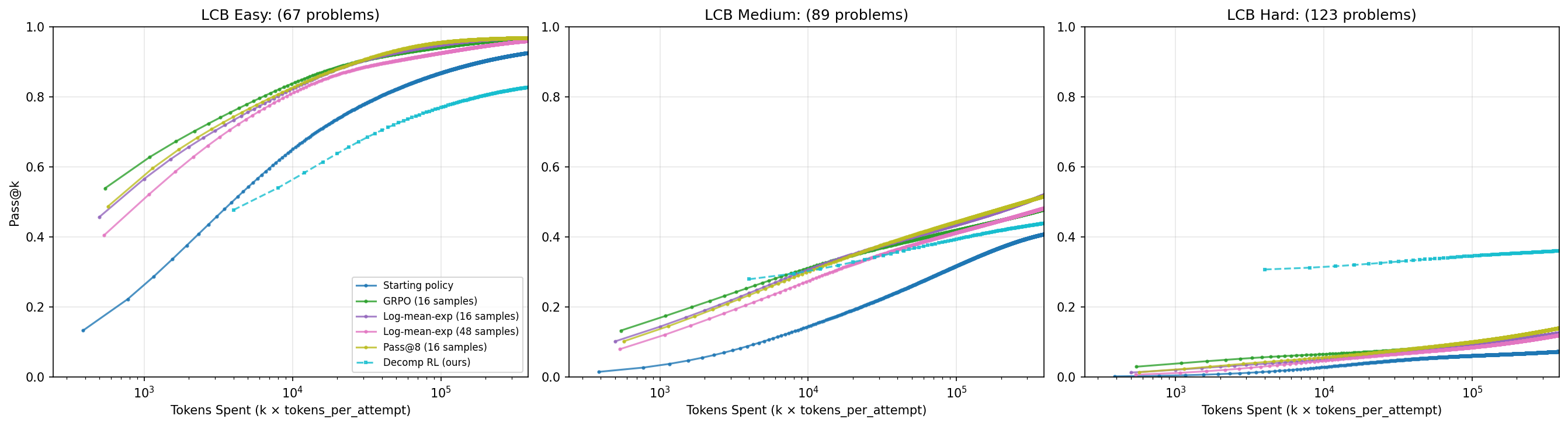}
    \end{minipage}%
    \caption{\textbf{Decomp RL solves harder problems than existing RL methods.} Our hierarchical training learns 1) a decomposition policy that breaks down problems into simpler functions and 2) an implementation policy that writes code for each function in parallel. Given a high token budget, Decomp RL solves up to $35 \%$ problems of the LiveCodeBench hard subset ($2024/08/01$ to $2025/02/01$) \cite{jain2024livecodebench}.}
    \label{fig:example}
\end{figure}

%% file: figures/old_fig1.tex
\begin{figure}[!h]
    \centering
    \includegraphics[width=0.7\linewidth]{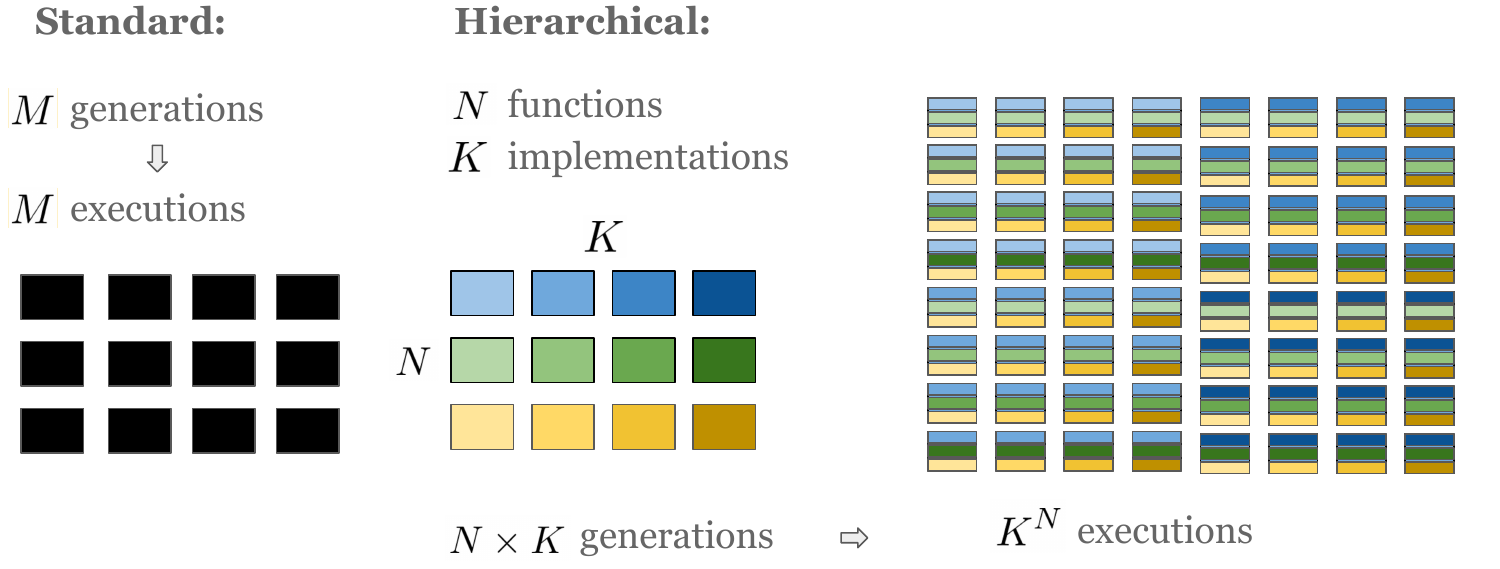}
    \caption{\textbf{Learning to scale CPU budget.} \textbf{(left)} Example problem from LiveCodeBench and the generated solutions with regular and hierarchical inference. \textbf{(right)} Solving a problem directly in $M$ tries leads to $M$ complete code solutions to evaluate. Hierarchical inference instead generates $N$ functions each implemented $K$ times independently. Recombining these partial code blocks gives up to $K^N$ full code solutions to evaluate.}
    \label{fig:old_example}
\end{figure}

%% file: method.tex
\label{sec:method}
We first introduce hierarchical generation in the context of problem solving (Section~\ref{sec:hierarchical}, \ref{sec:combination}) as described in Parsel \citep{zelikman2023parsel} and then show how we adapt this to train decomposition and implementation policies (Section~\ref{sec:rewards}).

\subsection{Hierarchical inference}
\label{sec:hierarchical}
We consider the task of generating code in a modular fashion, where model users have tighter control over the generation process than in the standard setup of whole-code generation from an autoregressive language model. Standard whole-code generation models the joint distribution of a piece of code of multiple parts $I_1, \ldots, I_n$ by means of its autoregressive decomposition and can optionally include a plan $D$ sampled before $I_1$. Suppressing conditioning on the problem $x$, we have:
\[
\pi(I_1, \ldots, I_n) = \pi(I_1 |D)\prod_{i=1}^n \pi(I_i \mid I_{i-1}, \ldots, I_1, D).
\]

In this work, we instead consider a hierarchical decomposition:
\[
\pi(D, I_1, \ldots, I_n) = \pi(D) \prod_{i=1}^n \pi(I_i \mid D)
\]
which models the implementations of the pieces $I_i$ as conditionally independent of each other given~$D$, and captures all interaction in the decomposition plan $D$.

In practice where we consider code generation tasks in Python, we represent decompositions $D$ by a set of functions specified by their signatures and natural language specifications in the form of docstrings. The parts $I_i$ are then implementations of those functions. We give an example of a decomposition in Figure~\ref{fig:example} and of a full hierarchical trajectory $(D, I_1, I_2)$ in the Appendix \ref{sec:example_decomp}. We use prompted and trained language models for the decomposition and implementation policies $\pi(D)$ and $\pi(I_i \mid D)$ (see Appendix~\ref{sec:prompts} for full prompts).

\subsection{Recombination}
\label{sec:combination}
Hierarchical inference comes with an important benefit: while the conditional independence assumption makes individual samples less precise, it can be exploited to generate samples cheaply by recombination. Here and in the following, let $D$ be a decomposition of size $n$ and $I_i^j, 1\leq i \leq n, 1 \leq j \leq k$, be $k$ implementations per function $I_i$ described by $D$. From these $nk$ implementation samples, we can create $k^n$ complete trajectories by considering all combinations $(D, I_1^{j_1}, \ldots, I_n^{j_n})$ for $1 \leq j_1, \ldots, j_n \leq k$. In other words, the number of complete trajectories scales polynomially in $k$ while the inference cost scales linearly.

This is particularly useful in domains where rewards are only observed for complete trajectories, which is the case for verifiable reasoning domains like mathematics and code. In such setups, verification is a cheap CPU operation while generation is an expensive GPU operation involving large language models. In this work, we consider the task of competitive programming where the model must produce a code solution to a problem in natural language. The reward is $1$ if the code passes all private tests (unknown to the model at inference time) and $0$ otherwise.

\subsection{Reinforcement learning: policy gradient estimators}
\label{sec:rewards}
As customary with large language models, we consider policy gradient reinforcement learning algorithms \citep{10.1007/BF00992696}.
We write $\tau = (D, I_1, \ldots, I_n)$ for a trajectory, $r(\tau)$ for the reward it receives and $\pi_{\theta}$ as the policy given by a large language model.
According to the policy gradient theorem \citep{sutton98book}, an unbiased estimate of the gradient $g$ of the expected reward is given by
\[
g = \nabla_\theta \mathbb{E}_{\tau \sim \pi_\theta}[r(\tau)]
= \mathbb{E}_{\tau \sim \pi_\theta} \left[
\sum_{a \in \tau} (r(\tau) - b) \nabla_\theta \log \pi_\theta(a)
\right],
\]
for any baseline $b$ which does not depend on the action $a \in \{D, I_1, \ldots I_n\}$.

We can now view recombination as a way to produce Monte-Carlo estimates of the policy gradient for hierarchical inference more efficiently.
Given $k$ implementations $I_i^j$ of the functions $I_i$ in the decomposition $D$, the standard estimator for $g$ used in sequential reinforcement learning is
\[
\hat{g}_\text{standard} = \frac{1}{k} \sum_{j \leq k}
\sum_{a \in \{D, I_1^j, \ldots, I_n^j\}} (r(D, I_1^j, \ldots, I_n^j) - b) \nabla_\theta \log \pi_\theta(a),
\]
which is optimal in the case of linear (autoregressive) rollouts where $I_k$ depends on $I_j$ for $k > j$.

For hierarchical inference, however, we made the modeling assumption that implementations are conditionally independent and all interactions are captured by their dependence on the decomposition $D$. In this case, we can evaluate all combinations of implementations generating a matrix $(I^{j}_i)_{j,i} \; \forall i,j$ for all functions. We can obtain the improved estimator:
\[
\hat{g}_\text{hierarchical} = \frac{1}{k^n} \sum_{j_1, \ldots, j_n \leq k}
\sum_{a \in \{D, I_1^{j_1}, \ldots, I_n^{j_n}\}} (r(D, I_1^{j_1}, \ldots, I_n^{j_n}) - b) \nabla_\theta \log \pi_\theta(a),
\]
of which $\hat{g}_\text{standard}$ represents the diagonal terms $(I^{j}_i)_{i,i} \; \forall i$. Clearly, both estimators are unbiased, but we have $\Var (\hat{g}_\text{hierarchical}) \leq \Var (\hat{g}_\text{standard})$, as we show in Theorem~\ref{thm:var}.

\subsection{Reinforcement learning objectives}
\label{sec:aggregation}
We do not expect hierarchical inference to improve upon whole code generation when compared on a single-attempt (pass@1) basis. This is because hierarchical inference comes with a loss in precision (conditional independence instead of joint modeling of implementations) and with only one sample we do not gain anything from the recombination. Instead, we propose to focus on the task of efficient inference scaling, i.e, comparing pass@k \citep{chen2021evaluating} metrics for large values of $k >> 100$.

Given $n$ rollouts $\tau_j$ with rewards $r = (r(\tau_1), \ldots, r(\tau_n))$ and a multi-sample objective function $f(r)$, \citet{tang2025optimizing} optimize ${E}_{\tau_i \sim \pi_\theta} [f(r)]$. Here, $f$ is typically a function of a variable number of arguments, permutation-invariant and monotonic in each argument such as the maximum function. They show this can be achieved directly by using gradients for any baseline $b$ independent of $\tau_i$ such as the leave-one-out baseline $b = f(r_{-i}) := f(r_1, \ldots, \hat{r}_i, \ldots, r_n)$ which omits $r_i$, giving
\[
g = \mathbb{E}_{\tau_i \sim \pi_\theta} \left[
\sum_i \left(f(r) - f(r_{-i})\right) \nabla_\theta \log \pi_\theta(\tau_i)
\right].
\]

\input{figures/transform}

In the case of hierarchical inference where the solutions are recombined before evaluation, we have a large number of correlated rewards and optimize for large pass@$k$. Unlike the original paper \citep{tang2025optimizing}, we can't take $f$ as the maximum over all samples since it would saturate and lead to extremely sparse advantages: a positive advantage would only be observed if $\tau_i$ is the only solution among $k$ attempts that solves the task (see Figure \ref{fig:baselines} in the Appendix). To adapt multi-sample objective training to this use case, we propose to use soft objective functions that interpolate between the hard maximum (pass@k training) and the average (standard training). We opt for the $\lme_\beta$ function defined by
\begin{equation} \label{eq:lme}
    \lme_\beta(r_1, \ldots, r_n) = \beta \log \left( \frac{1}{n} \sum_{i=1}^{n} e^{r_i / \beta} \right).
\end{equation}

This function smoothly interpolates between the average ($\beta \to \infty$) and the maximum ($\beta \to 0$), and for a well-chosen $\beta$ closely approximates a log-uniform mixture of pass@$k$ objectives (see Figure~\ref{fig:transform}), balancing exploration and exploitation (see Appendix~\ref{app:lme_properties}).

Concretely, in the setup of hierarchical inference, we apply this form of reward aggregation as follows. We generate $d$ decompositions $D_i$ of size $n_i$, $1 \leq i \leq d$, respectively and $k$ implementations for each function in each $D_i$. For each decomposition, we evaluate $m \leq k^{\max{(n_1,..,n_d})}$ combinations. Write $r$ for all $dm$ rewards obtained in this way, $r_{ij}$ for a single reward, $i \leq d, j \leq m$ and $A(r_{ij})$ for the set of actions involved in $r_{ij}$.

DecompRL is a cooperative framework of two policies: decomposition $\pi_\theta(D)$ and implementation $\pi_\theta(I \mid D)$. For the decomposition policy, we use the policy gradient
\begin{equation} \label{eq:objective}
    g = \mathbb{E}_{r_{ij}} \left[
\sum_{i=1}^{n} \left(\lme_\beta(r) - \lme_\beta(\{ r_{i'j} \mid i' \neq i) \})\right) \nabla_\theta \log \pi_\theta(D_i)
\right].
\end{equation}

For the implementation policy, assume $d=1$ and write $r_j$ for $r_{1,j}$. Then we use the policy gradient
\[
g = \mathbb{E}_{r_j} \left[ \sum_{l \leq n, j_l \leq k} \left(\lme_\beta(r) - \lme_\beta(\{ r_j \mid I_l^{j_l} \notin A(r_j) \})\right) \nabla_\theta \log \pi_\theta(I_l^{j_l} \mid D)
\right].
\]

In other words, the advantage is computed from the multi-sample objective using all observed rewards and baselined by the multi-sample objective computed from all rewards that an action did not participate in.
Because, by this construction, the baseline is independent of the action, it does not bias the gradient estimate.
The leave-one-out baseline reduces the overall gradient variance (see Appendix~\ref{app:var_reduction} for the computation adapted to multi-sample objectives). Empirically, the reward tensor over recombinations is well-approximated by a rank-1 (Appendix~\ref{app:reward_tensor}), confirming that credit concentrates along individual function axes and that the leave-one-out baseline provides reliable advantage estimates from a sparse subset of $k^n$ combinations.

%% file: figures/transform.tex
\begin{figure}[t]
\centering
\includegraphics[width=\textwidth]{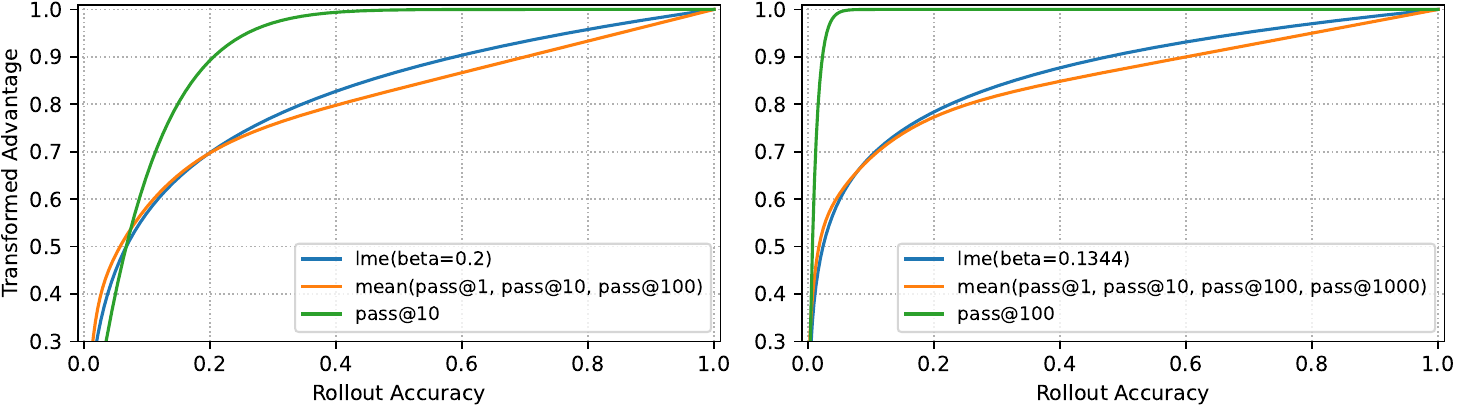} 
\caption{{Nonlinear reward transformations with utility functions.} Reward transformation induced by $\lme_\beta$ for different values of $\beta$, alongside the ones of pass@k: $u(p) = 1 - (1-p)^k$, and the mean of pass@k values. See Appendix~\ref{app:transform} for details.}
\label{fig:transform}
\end{figure}

%% file: experiments.tex
\label{sec:exp}

\subsection{Overview}
\textbf{Datasets and models: } We perform experiments with three different large language models: Qwen 2.5 7B \citep{qwen2025qwen25technicalreport}, Llama 3.1 8B Instruct \citep{meta2024llama} and the Code World Model 32B \citep{cwm}. We train all models with online reinforcement learning on a mix of $15,000$ competitive programming problems including the CodeContest \citep{li2022alphacode} and TACO \cite{li2023tacotopicsalgorithmiccode} training sets. We evaluate on the DeepMind Code Contest validation set \citep{li2022alphacode} (117 problems) and LiveCodeBench range $2024/08/01$ to $2025/02/01$ (279 problems) \citep{jain2024livecodebench}. All evaluations are done with temperature $1.0$ and top-p $1.0$ to promote answer diversity.

\textbf{Algorithms:} Similar to \citet{deepseek_grpo,deepseek_R1}, we optimize policy gradients using Proximal Policy Optimization (PPO) \citep{ppo} without a critic model using the advantage described in section \ref{sec:rewards}. We use an asynchronous distributed RL framework \citep{RLEF,rlhf_faster} where we divide $80$ H100s into a set of workers (producing samples) and trainers (updating the current policy). We use $72$ workers and $8$ trainers. All code generations are evaluated on an external CPU cluster.

\textbf{Baselines: } We compare with
Group Relative Policy Optimization (GRPO) and algorithms optimized for high pass@$k$ performance specifically: pass@$k$ training \citep{chen2025pass} (set $k = 8$), Soft Policy Optimization (SPO \cite{cohen2025soft}) and our own $\lme$ leave-one-out objective (see Equation \ref{eq:objective}) with $\beta = 0.3$. All methods use $16$ samples per prompt. We also include a $\lme$ objective with $48$ samples per prompt and $\beta = 0.1$ to mimic the number of forwards per decomposition present in DecompRL training (at most $6$ functions and $8$ implementations per problem). See hyperparameter and baseline tuning details in Appendix~\ref{sec:baselines} and \ref{app:hyper}.

\subsection{DecompRL as Multi-Agent Reinforcement Learning}
\label{sec:multi-rl}
DecompRL is a form of cooperative multi-agent RL where a decomposition and an implementation policy maximize the number of correct code combinations. Multi-agent reinforcement learning poses challenges not present in single-agent reinforcement learning. Agent policies are interdependent and non-stationary, which raises the questions of how to perform credit assignment after jointly observed rewards and how to optimize toward a social optimum (in the cooperative case) \citep{marl-book}. To address the moving target problem, we use sequential training and a counterfactual action value estimation \citep{foerster2024counterfactualmultiagentpolicygradients} with our leave-one-out baseline. Similar to Expectation-Maximization algorithms \cite{rosipal01}, our training has two stages. The decomposition and implementation policies are represented by two different copies of our model. First, we train the decomposition policy for $30k$ steps keeping the implementation policy fixed. Second we train the implementation policy also for $30k$ steps keeping the decomposition fixed. For each problem, we use $8$ decomposition samples. For a decomposition with $n$ functions ($n_{max} = 6$), we implement each $8$ times yielding $8^n$ code combinations. We randomly sample $512$ to evaluate out of $8^n$ to balance exploration between breadth-first (exploring new problems) and depth-first search (scaling evaluations per problem). Although we have two copies of the model, only one is trained at a time so this has similar cost as the reference policy used in traditional RL algorithms.

%% file: results_good.tex
DecompRL produces large numbers of diverse samples that lead to diverse solutions (Section~\ref{sec:max_eval}), learns modularity (Section~\ref{sec:recomb}) and solves new problems (Section~\ref{sec:sota})

\begin{table}[ht]
\centering
\caption{\textbf{DecompRL makes the most of high token compute.} We compare pass@tokens: given a set token budgets what is the solve rate across different online RL training methods on LiveCodeBench using Qwen 2.5 7B.}
\begin{tabular}{@{}lcccccc@{}}
\toprule
Token budget & GRPO & instruct & lme16 & lme48 & pass@8 & DecompRL (ours) \\ \midrule
1,000 tokens   & 0.18 & 0.06 & 0.19 & 0.12 & 0.15 & 0.18 \\
5,000 tokens   & 0.29 & 0.16 & 0.28 & 0.26 & 0.27 & 0.18 \\
10,000 tokens  & 0.32 & 0.21 & 0.31 & 0.30 & 0.31 & 0.25 \\
50,000 tokens  & 0.38 & 0.30 & 0.38 & 0.36 & 0.39 & \textbf{0.40} \\
100,000 tokens & 0.40 & 0.33 & 0.40 & 0.39 & 0.41 & \textbf{0.44} \\
500,000 tokens & 0.44 & 0.39 & 0.46 & 0.44 & 0.46 & \textbf{0.48} \\ \bottomrule
\label{tab: lcb_qwen}
\end{tabular}
\end{table}

\subsection{DecompRL leads to diverse solutions}
\label{sec:max_eval}
\input{figures/maxeval}
A decomposition with $n$ functions and $k$ implementations per function costs $nk+1$ model forwards and yields $k^n$ (correlated) code combinations for evaluation.
For hierarchical inference to make a difference, we need the final solutions to have enough diversity to offset their generation cost compared to $nk+1$ independent whole-code samples. With the Llama 3.1 8B and Qwen 2.5 7B, we show training the decomposition policy increases success rates per decomposition and number of evaluations.

In Figure~\ref{fig:max_eval}, we investigate the diversity of recombined programs by fixing a maximum number $m$ of code evaluations per decomposition and observing the relationship between $m$ and the overall success rate of a DecompRL model.
It can be seen that the success rate from less than 50 model forwards per problem shows no sign of plateauing even when scaling $m$ up to 1000 evaluations. Crucially, this remains true throughout reinforcement learning training as shown in Figure \ref{fig:exploration} where the solve rate continues to increase given a fixed evaluation budget of $m = 512$ code samples.

\subsection{Recombination matters}
\label{sec:recomb}
\input{figures/scatter}
The benefit of recombining partial solutions is most pronounced when the decomposition size $n$ is large where $n$ is the number of functions picked by the model.
In our experiments, we see that models with similar success rates using 10 decompositions behave very differently when inference is scaled to 100 decompositions. As we scale the inference budget per problem, models that create larger decompositions improve faster than those with smaller ones. The success rate with 10 decomposition attempts does not predict the success rate with 100 attempts, but the size of the decomposition does (see Figures~\ref{fig:scatter_passrate} and \ref{fig:scatter_size}).

\subsection{Solving new problems}
\label{sec:sota}

DecompRL through recombination and diverse evaluations solves new problems that are harder (Figure \ref{fig:example}) and only solvable at high inference budgets (Figure \ref{fig:scaling} and Table \ref{tab: lcb_qwen}). Using different models with DecompRL training, we beat the best existing methods for reinforcement learning when using high inference budgets (up to 3 million tokens per problem). We evaluate $m = 4096$ combinations per decomposition, which is roughly 80 times more than $nk+1$ language model calls would have produced with standard whole-code generation. On the other hand, $m = 4096$ is small compared to the theoretical maximum $k^{n_{\mathrm{max}}} = 8^6$  of combinations that hierarchical inference can produce.
Additional CPU execution resources could further increase the appeal of the method according to the scaling behavior in Figure~\ref{fig:scaling}. DecompRL can help overcome the limits of our SFT checkpoint and find solutions not available with regular inference. We see it as a useful way to gather offline data to improve the starting policy.

%% file: figures/maxeval.tex
\begin{figure}[t]
\centering
\begin{subfigure}[b]{0.31\textwidth}
    \centering
    \includegraphics[width=\textwidth]{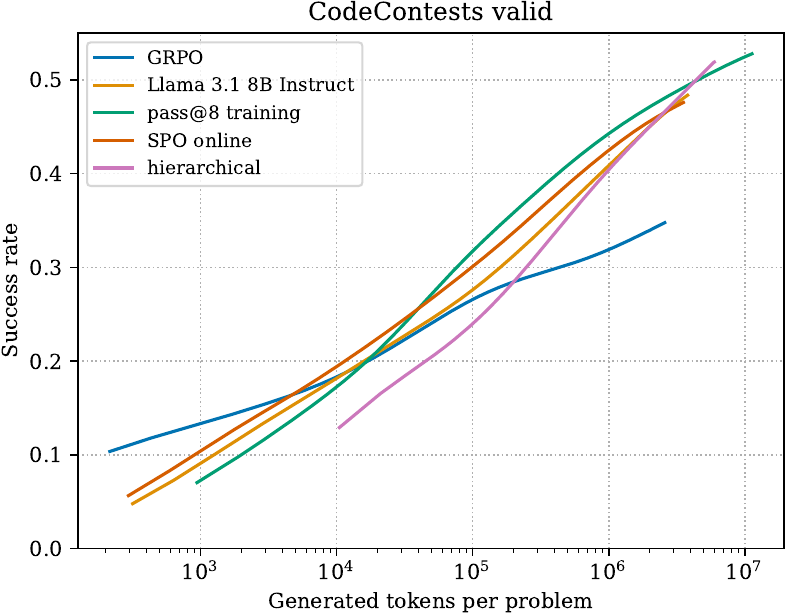}
    \caption{}
    \label{fig:dmc_valid}
\end{subfigure}
\hfill
\begin{subfigure}[b]{0.3\textwidth}
    \centering
    \includegraphics[width=\textwidth]{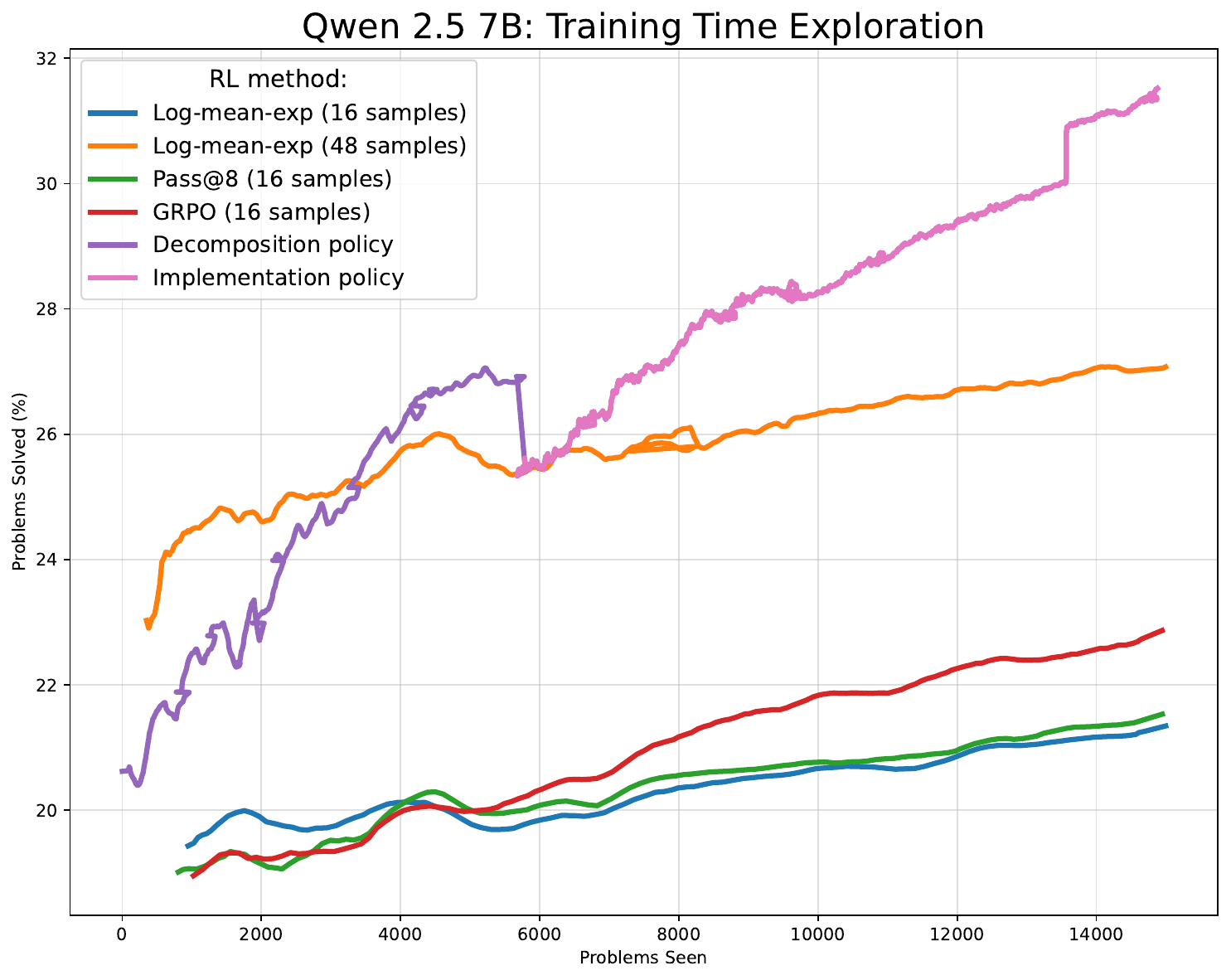}
    \caption{}
    \label{fig:exploration}
\end{subfigure}
\begin{subfigure}[b]{0.325\textwidth}
    \centering
    \includegraphics[width=\textwidth]{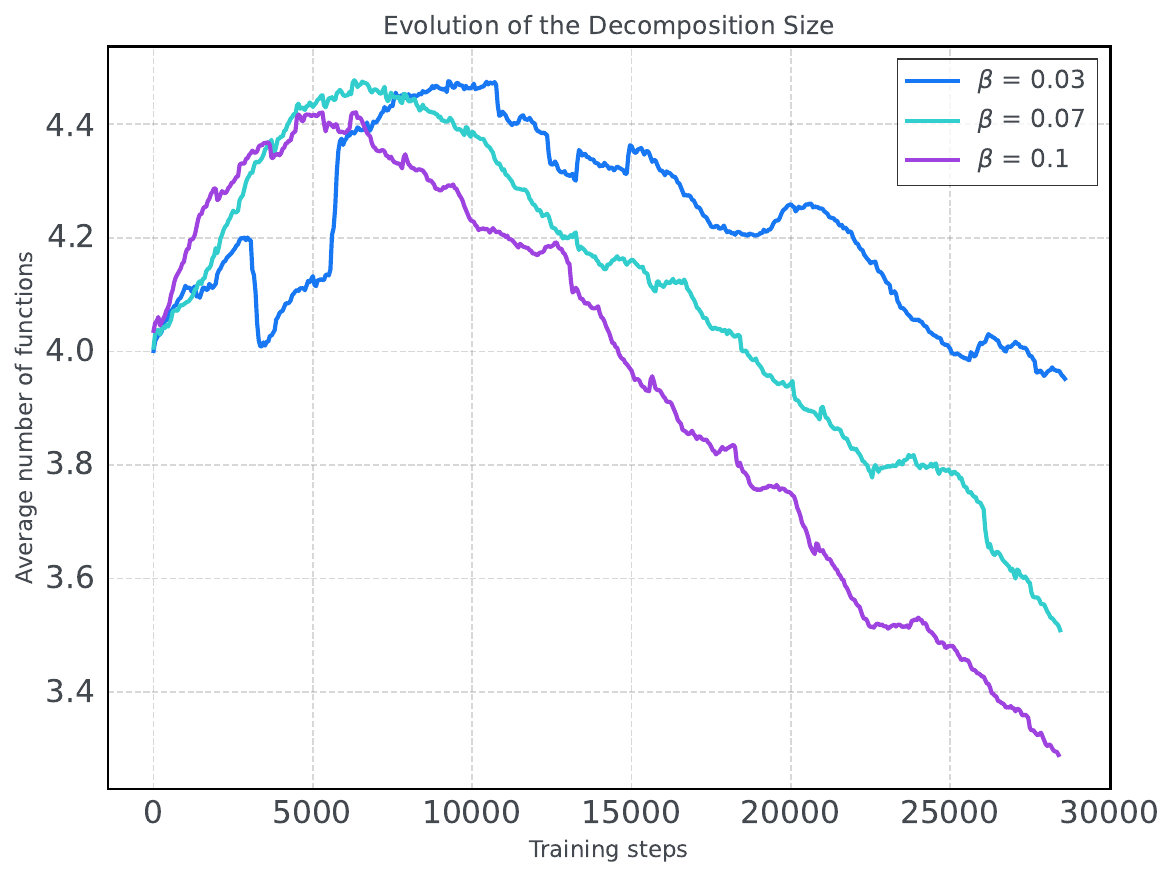}
    \caption{}
    \label{fig:decomp_decay}
\end{subfigure}
\caption{\textbf{Diversity from recombined samples.} \textbf{(a)} Scaling inference-compute on CodeContests valid: Standard reinforcement learning for code (GRPO) increases pass@1 while degrading pass@k for large $k$. Our method of hierarchical inference is competitive with scaling inference on the instruct model as well as with diversity-focused  reinforcement learning methods pass@8-training and SPO when evaluating only up to 4096 combinations per decomposition.
\textbf{(b)} Training the implementation and decomposition policies both help increase the solve rate per problem. \textbf{(c)} The decomposition policy learns to use less functions during training. We apply a $\lme$ normalization where the $\beta$ can be tuned between mean ($\beta \to \infty$) and max ($\beta \to 0$).}
\label{fig:scaling}
\end{figure}

%% file: figures/scatter.tex
\begin{figure}[t]
    \centering
    \begin{subfigure}[b]{0.3\textwidth}
        \includegraphics[width=\textwidth]{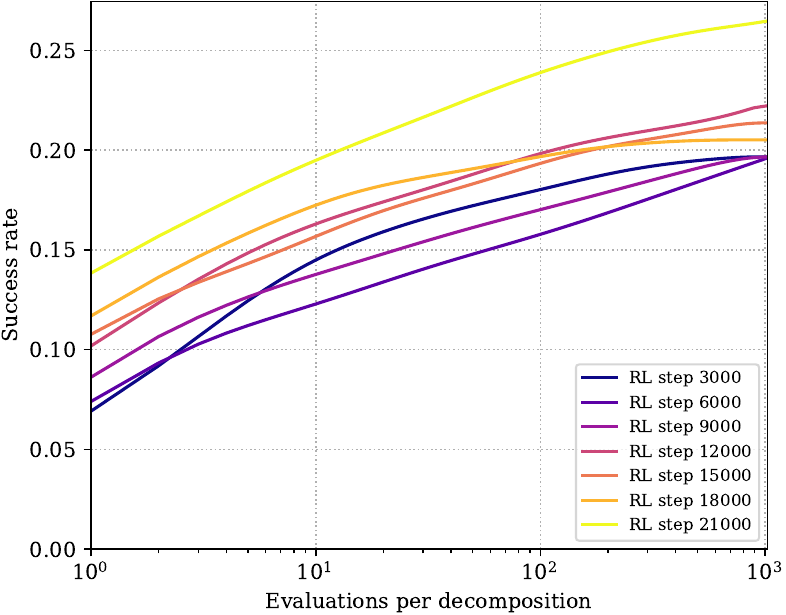}
        \caption{}
        \label{fig:max_eval}
    \end{subfigure}
    \hfill
    \begin{subfigure}[b]{0.3\textwidth}
        \centering
        \includegraphics[width=\textwidth]{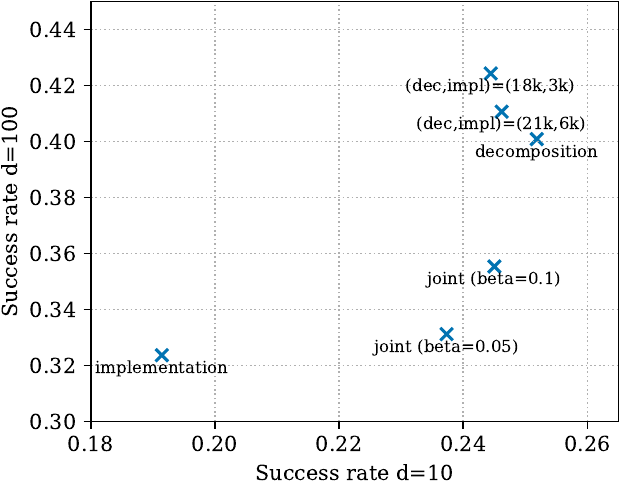}
        \caption{}
        \label{fig:scatter_passrate}
        
    \end{subfigure}
    \hfill
    \begin{subfigure}[b]{0.3\textwidth}
        \centering
        \includegraphics[width=\textwidth]{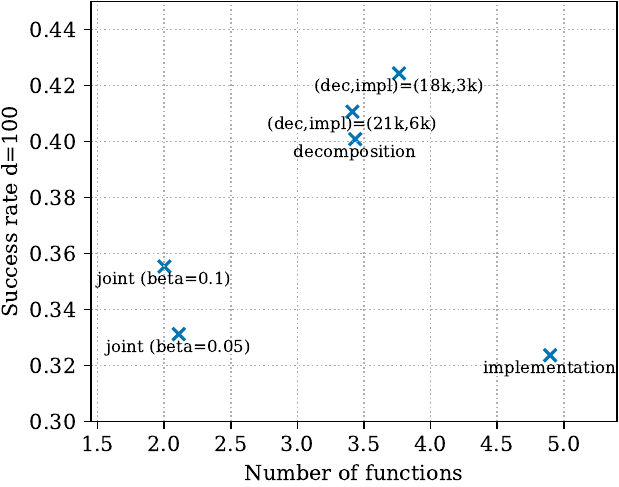}
        \caption{}
        \label{fig:scatter_size}
    \end{subfigure}
    \caption{\textbf{Recombination is critical for benchmark performance.}
A series of analysis done with LLama 3.1 8B Instruct.
\textbf{(a)} Training a decomposition and implementation policy helps maximize success rate per re-combined code evaluation despite a fixed inference budget. 
\textbf{(b)} Let $d$ the number of decompositions per problem, success rate at $d=10$ attempts does not predict success rate at $d=100$ attempts,
\textbf{(c)} whereas the size of decompositions does. Best performing models for large attempt numbers produce bigger decompositions on average.
}
    \label{fig:scatter}
\end{figure}

%% file: ablations.tex
\subsection{Ablation Study on DecompRL}
In standard RL with LLMs, one action leads to one reward and with GRPO we compute the advantage over a group $G$ of rewards where $G << 100$ due to inference limitations. In DecompRL, we instead have two policies performing $k \times n$ actions which lead to $k^n$ rewards. The advantage of each action is therefore estimated by many indirect and correlated rewards. In order to prevent over-estimating certain actions, we introduce in equation \ref{eq:objective}: 1) a leave one out instead of mean baseline, 2) transforming the rewards per action using a logmeanexp function (Equation \ref{eq:lme}) 3) training sequentially a set of two policies. Figure \ref{fig:ablations} shows how each component prevents instabilities. We run ablations on the Code World Model 32B.

\begin{figure}[t]
\centering
    \includegraphics[width=\textwidth]{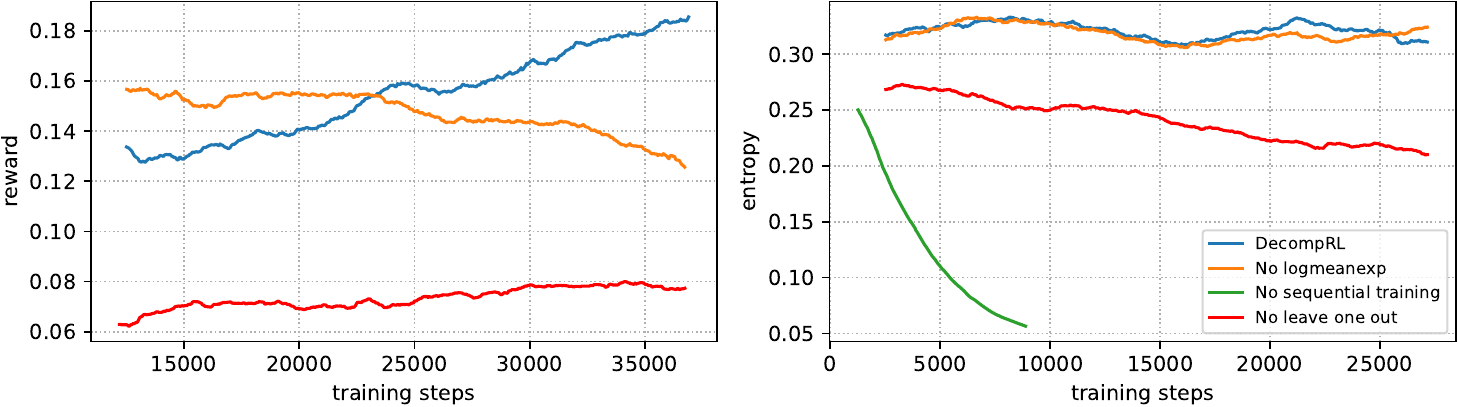}
\caption{\textbf{Ablation study on the DecompRL objective,} removing logmeanexp for a mean aggregation over rewards using a softmax, removing sequential training for joint training of both policies, and the leave one out baseline for value estimation in the advantage all lead to entropy or reward collapse.}
\label{fig:ablations}
\end{figure}

\textbf{Ablation on logmeanexp.} We compare our max interpolation (logmeanexp, $\beta = 0.03$) with a mean interpolation (softmax, $\beta = 0.5$). As we have many noisy rewards optimizing towards the mean eventually leads to reward collapse. Lower $\beta$ in $\lme$ also leads to faster decomposition size decay (see Figure \ref{fig:decomp_decay}).

\textbf{Ablation on sequential training.} Instead of alternating between training the implementation and decomposition policies, we study two other strategies: training only the implementation policy, and jointly training both policies (meaning no frozen model). Without clear credit assignment and with an over representation of implementation samples, we overfit the implementation policy and underfit the decomposition policy. Joint training leads to entropy collapse (Figure \ref{fig:ablations}) whereas sequential training allows us to keep scaling the number of functions with increased gains in performance (Figure \ref{fig:scatter}).

\textbf{Ablation on leave one out baseline.} Removing the leave one out baseline in the advantage calculation for both policies leads to much lower rewards and earlier plateau.

\subsection{Computational cost: GPU to CPU bottleneck shift}
\label{sec:cost}
A key advantage of DecompRL is that it shifts the dominant cost from GPU inference to CPU evaluation. Table~\ref{tab:wallclock} reports wall-clock breakdowns for Qwen~2.5~7B on our training setup (see details in Appendix~\ref{app:hyper}).

\begin{table}[ht]
\centering
\caption{\textbf{Wall-clock cost breakdown per training step} for Qwen~2.5~7B. DecompRL reduces GPU token generation by $\sim$50$\times$ at 512 evaluations, shifting the bottleneck to CPU-bound code execution. These numbers are per prompt so for a full sample size.}
\label{tab:wallclock}
\begin{tabular}{@{}lc cc cc c cc@{}}
\toprule
& & \multicolumn{2}{c}{Generation} & \multicolumn{2}{c}{Execution} & Training & \multicolumn{2}{c}{Total} \\
\cmidrule(lr){3-4} \cmidrule(lr){5-6} \cmidrule(lr){7-7} \cmidrule(lr){8-9}
Method & Samples & Tokens & Time (s) & Calls & Time (s) & Time (s) & Time (s) & GPU / CPU \\ \midrule
Standard & 16 & 10k & 108 & 16 & 371 & 1.1 & 479 & 23\% / 77\% \\
Standard & 512 & 198k & 2100 & 512 & 11878 & 21.8 & 14000 & 15\% / 85\% \\
DecompRL & 512 & 4k & 42 & 512 & 11878 & 0.4 & 11900 & 0.4\% / 99.6\% \\
\bottomrule
\end{tabular}
\end{table}

At 512 evaluations per problem, DecompRL generates only $\sim$4k tokens (one decomposition plus $n \times 8$ short function implementations) versus $\sim$198k for standard sampling, a $\sim$50$\times$ reduction in GPU cost. The evaluation cost is identical in both cases since both evaluate the same number of candidate solutions. This makes the total wall-clock time dominated by CPU execution rather than GPU generation.

When parallelized, the wall-clock saving from DecompRL scales with how GPU-bound the system is: at 8 GPUs with 128 remote execution threads, where generation dominates, DecompRL delivers a 3.6$\times$ speedup; at 8 GPUs with 64 threads, a 2.3$\times$ speedup; while at 72 GPUs with 128 threads, where remote execution already dominates, the saving shrinks to 1.3$\times$. Since our cluster is heavily GPU-optimized, we argue scaling the number of CPUs (which is often cheaper than scaling GPUs) could unlock even larger gains for DecompRL since the marginal cost of evaluating more recombinations is essentially free on the GPU side.

%% file: limits.tex
\label{sec:limits}
\paragraph{Format tax.} Hierarchical inference without RL training underperforms standard inference, with pass@1 dropping by 3 percentage points for Qwen~2.5~7B and performing similarly below the standard baseline for CWM~32B (Table~\ref{tab:hierarchical_baselines}). DecompRL with Llama 3.1 8B Instruct model degrades performance from $26.6 \%$ to $14.5 \%$ on the CodeContest validation with the same token budget (10 attempts with DecompRL, corresponding to pass@$316$ for standard inference). On the ``easy" split of LiveCodeBench, DecompRL performs worse than the starting policy using Qwen 2.5 7B (see Figure \ref{fig:example}). More generally, for easy problems and low token budgets, the format tax of decomposition outweighs the gains from additional recombined evaluations, since a strong base model can already solve these problems in a single monolithic attempt. Looking at the CodeContest training set as an example of competitive programming data, we notice decomposition and function generation are off-policy tasks for the model. Only 65~\% of the provided solutions in the CodeContests training set contain functions. We expect data in the training stage of the Llama 3.1 8B Instruct model to follow a similar distribution and potentially not encourage modular coding. The model has also been heavily fine-tuned through RLHF \citep{christiano2023deepreinforcementlearninghuman}, human annotations and synthetic data for the specific task of generating full code answers when prompted with problems in the standard inference setup. DecompRL unlike other reinforcement learning algorithms has to overcome this off-policyness during training. RL training closes this gap, with DecompRL surpassing even standard inference for both models (Table~\ref{tab:hierarchical_baselines}).

\paragraph{Size collapse.} Hierarchical inference comes with the cost of reduced precision at the single sample level. Decomposition training leads to a decreasing decomposition size (see Figure \ref{fig:decomp_decay}). This can allow for more relevant functions but in the limit if we have a single function, hierarchical inference becomes equivalent to classical whole-code generation. We see this as a form of reward hacking during training which can be delayed by using a lower $\beta$ in the $\lme$ objective and alternating between implementation and decomposition training. Simpler approaches such as bonuses based on the decomposition size led to reward hacking.

\paragraph{Sampling time.} During training, DecompRL scales the number of evaluations per problem to up to $512$ in our experiments compared to the default $16$. This allows for more exploration during training (see Figure \ref{fig:max_eval}) but can also lead to: slower worker GPUs and introduce off-policyness in online RL. We see DecompRL as a policy distillation method for gathering new solutions on the training set not available with regular RL. If we have a good enough starting policy using regular RL remains more efficient as a post training strategy.

%% file: related_works.tex
\paragraph{Scaling inference methods.}
Intermediate steps to help models generate correct code solutions has been studied at the ``plan'', ``algorithm'' and ``function'' level \citep{khot2022decomposed, self_discover, zhou2022least, jain2023llm, plansearch}. We adapt the decomposition and function approach from \citet{zelikman2023parsel}. To promote sampling diversity, optimizing over code generation trajectories has been studied mainly with repeated sampling \citep{li2022alphacode}, self-refinement chains with execution feedback \citep{shinn2023reflexion, madaan2023self}, and tree search using multiple children solutions per chain \cite{tree_search, ToT, light2025sfs}. Tree based generation systematically organizes and potentially directs linear generation, but generally does not create a combinatorial increase in solutions to evaluate (each evaluated leaf has to have been generated).

\paragraph{Recombination of partial solutions.}
Methods that decompose a problem into separate parts which can be solved independently and then recombined have been studied for solving mathematical and coding tasks. In mathematical theorem proving, tree search in formal theorem proving systems decomposes the problem into sub-goals that can be solved and checked independently by the verifier \citep{lample2022hypertree,polu2022formal}. In code generation and symbolic regression, library learning creates solutions to sub-problems for later reuse \citep{ellis2023dreamcoder,gauthier2023alien,grayeli2024symbolic}. These methods study re-combination at the inference level with a fixed policy for determining blocks whereas we propose reinforcement learning to improve the policy. Programs can be analyzed hierarchically \citep{shi2024code} or generated hierarchically \citep{zelikman2023parsel}. Via clustering in the embedding space (CodeChain \citep{le2023codechain}), prompting different agents MapCoder \citep{mapcoder} and iterating on trees of functions based on shared consensus (Divide and Conquer \citep{divide_conquer}), inference methods have enforced modularity in code generation trajectories.

\paragraph{Self-play reinforcement learning.} In formal mathematical proving, a \emph{conjecturer} model asserts statements for a prover to prove, either as a self-guided exploration task with intrinsic motivation \citep{poesia2024learning} or as a live augmentation technique during a standard reinforcement learning run for theorem proving \citep{dong2025stpselfplayllmtheorem}. In the domain of code, self-play involves a programming puzzle generator and a model that solves the generated tasks \citep{haluptzok2022language,teodorescu2023codeplay}. In DecompRL, the decomposition model can be viewed as a self-play teacher that is conditioned on the underlying programming problem to be solved.

\paragraph{Expert Iteration.} Expert iteration \citep{anthony2017thinking} divides the problem into an ``expert'' (a slow tree search algorithm doing the heavy exploration) and an ``apprentice'' (a fast neural network). As DecompRL is optimized for harder problems and high pass@$k$, we can view it as the ``expert'' policy gathering novel solutions to train a starting RL policy offline (the ``apprentice"). Online expert iteration methods such as STaR \citep{Star}, Quiet-STaR \citep{zelikman2024quiet}, ReST \citep{rest}, and ReST-MCTS \cite{zhang2024rest} use the same model as expert and apprentice to iteratively bootstrap its reasoning capabilities using successful trajectories. Similar to our hierarchical inference process, ReST-MCTS goes beyond regular inference using tree search to find solutions normally missed by a weak starting policy.

%% file: discussion.tex
Solving programming problems by decomposing them into simpler subproblems and recombining them into full solutions allows scaling inference-time evaluation budgets with a limited number of language model calls.
However, off-the-shelf instruction finetuned models do not fare well at this task.
We introduce DecompRL, a reinforcement learning technique that directly trains models for hierarchical inference, greatly enhancing their performance on this task.
This method encourages diversity by directly optimizing for a mixture of pass@k objectives with a novel $\lme$ advantage function.
The resulting models use recombination to solve problems unreachable with regular RL training. With evaluation up to pass@1k or $10^7$ generated tokens per problem, we show with different models (Llama 3.1 8B Instruct, Qwen 2.5 7B) that the success rate of DecompRL continues to scale whereas other methods begin saturating. With larger CPU resources, we envision the potential for significant advancements when scaling such systems to millions of evaluated answers. Moreover, the modularity of the resulting generation pipeline offers avenues for interpretability insights and an axis of control which is not available in other methods of scaling inference compute, such as in repeated sampling or scaled chain-of-thought.

A promising direction is to use DecompRL as the explorer stage of an explorer/distiller pipeline \citep{guo2025deepseek, rastogi2025learn}: DecompRL generates diverse correct solutions at high budgets, and a standard pass@1 policy is then distilled from these trajectories via supervised fine-tuning. This two-stage protocol could combine DecompRL's exploration strength with the efficiency of monolithic generation at deployment time. We also note that problem types that are naturally parallelizable or have a hierarchical structure such as multi-step proofs, nested algorithmic routines, or layered engineering systems are likely to benefit most from the modular decomposition that DecompRL encourages.

Looking forward, many scientific and engineering tasks are more easily evaluated than solved, presenting opportunities for innovative solutions. For writing entire codebases that solve specific problems, models could achieve acceptable pass rates at scale (e.g., pass@1M), even if their initial performance (pass@1) is poor. Our findings suggest that DecompRL is a high performing search policy to solve new and harder problems and could become an ingredient of any training pipeline where recombination is possible and evaluation is cheap while generation is expensive.

More generally, our findings shed light onto the hard exploration question in LLM post-training \citep{cui2025entropymechanismreinforcementlearning}: How can models discover new behaviors that allow solving previously unsolved tasks without the need for human annotations, and more efficiently than via finetuning on large-scale rejection sampled datasets \citep[e.g. in formal math][]{polu2022formal,lin2025goedelproverfrontiermodelopensource}? While DecompRL is specific to code generation, it hints at a possible future where exploration policies are trained specifically for creating high-quality data with behaviors that the final, exploiting policy trained with a standard pass@1 objective could not have found by itself over the course of its training.

%% file: appendix.tex
\section{Proofs}
\input{proof}
\label{sec:proof}

\section{Information-Theoretic View on Inference Scaling}
\label{app:gumbel}
Consider a dataset $D$ of problems $x \in D$, an evaluation function $v(x, y) \in \{0,1\}$, a probabilistic model $\pi$ and consider the pass rate or coverage function
\[
c(k) = \E_{x \sim D, y_1, \ldots, y_k \sim \pi(x)}[\max_i v(x, y_i)].
\]
Empirical investigations \citep{brown2024large,hughes2024bestofnjailbreaking,openai2023gpt4} frequently report power law scaling of the logarithmic pass rate with respect to the number of samples of the form
\[
\log(c) \approx -ak^{-b},
\]
where $a, b > 0$.

To analyze this relationship, note that $k$ can be seen as an expected waiting time in the Bernoulli process with repeated samples from $v(x, \pi(x)) \sim \mathrm{Bern}(p)$ where $p = \frac{1}{k}$, and we can consider the information quantity $h = -\log p = \log k$. Under the logarithmic power law model, we know the cumulative distribution function
\[
\Prob(k \leq k_0) = c(k_0) = \exp \left( -ak^{-b} \right)
\]
and thus for $h_0 = \log k_0$,
\[
\Prob(h \leq h_0) = \exp \left( -ae^{-bh} \right).
\]
Differentiating, we get the probability density function
\[
p(h) = ab \exp(-bh -ae^{-bh}) = \frac{1}{\beta} \exp \left( -\frac{h-\mu}{\beta} - e^{-\frac{h-\mu}{\beta}} \right),
\]
i.e. a Gumbel distribution with mode parameter $\mu = -\frac{\log a}{b}$ and scale parameter $\beta = \frac{1}{b}$.

Note that the Gumbel distribution arises as an extreme value distribution over exponentially distributed samples, and we can interpret this by imagining several ``latent dimensions of problem difficulty'' $h_i$, the largest of which determines the overall difficulty of a given problem: $h = \max_i h_i$.

\section{Properties of the Logmeanexp Objective}
\label{app:lme_properties}
The $\lme_\beta$ function (Equation~\ref{eq:lme}) has several theoretical appeals beyond serving as a smooth interpolation between the average ($\beta \to \infty$) and the maximum ($\beta \to 0$):
\begin{itemize}
    \item For a well-chosen value of $\beta$, it closely approximates the average between pass@1 + pass@10 + pass@100 + pass@1000. Such objectives strike a balance between favoring exploration and diversity (due to pass@1000 components) and re-ranking to favor successful solution attempts for exploiting the current knowledge (due to pass@1 components). See Appendix~\ref{app:transform} for details and Figure~\ref{fig:transform}.
    \item It represents a soft form of optimism over the other actions. Recall that Deep-Q-Learning \citep{mnih2013playing} uses $R_{t} + \max_a Q(s_{t+1}, a)$ as the Q-value target for $Q(s_t, a_t)$ if the action led to the new state $s_{t+1}$. This target can be seen as an approximation of $Q^\ast(s_t, a_t)$ for the $Q^\ast$ the state-action values of the optimal policy $Q^\ast$. In the language of soft reinforcement learning \citep{cohen2025soft}, $\mathcal{V}_\beta = \lme_\beta(r_1, \ldots, r_n)$ can be seen as the $\beta$-soft value function of selecting one out of $n$ rollouts with uniform prior $\pi_0$.
\end{itemize}

\section{Reward Aggregation as Utility Functions}
\input{figures/reweight}
\label{app:transform}
Consider as in Section~\ref{sec:aggregation} $n$ rewards $r_1, \ldots, r_n$. For zero-one rewards $r_i \in \{0, 1\}$, and different choices of a multi-sample objective function $f$, we write $c = \sum_i r_i$ and factor $f(r_1, \ldots, r_n) = u(p)$ via $p = \frac{c}{n}$. Note that $p$ is the mean estimator of $\E[r]$ if $r_i$ are samples from the distribution of $p$.

\paragraph{Logmeanexp.} We compute
\[
f(r_1, \ldots, r_n) = \beta \log \frac{1}{n} \sum_i e^{r_i/\beta}
= \beta \log \frac{1}{n} \left(ce^{1/\beta} + (n-c) \right)
= \beta \log \left(pe^{1/\beta} + 1-p \right).
\]
\paragraph{Pass@k with replacement.} In this case, $f(r_1, \ldots, r_n)$ is already defined by
\[
f(r_1, \ldots, r_n) = 1 - \left(1 - \frac{c}{n} \right)^k = 1 - (1-p)^k.
\]
\paragraph{Softmax-weighted average.} \citet{tang2025optimizing} suggest the multi-sample objective $\sum_i \frac{e^{r_i/\beta}}{\sum_j e^{r_j/\beta}} r_i$. In this case, we obtain
\[
f(r_1, \ldots, r_n)
= \sum_i \frac{e^{r_i/\beta}}{\sum_j e^{r_j/\beta}} r_i
=  \frac{ce^{1/\beta}}{c e^{1/\beta} + (n-c)}
=  \frac{pe^{1/\beta}}{p e^{1/\beta} + (1-p)}
= 1 - \frac{1-p}{pe^{1/\beta}+1-p}.
\]

Utility functions are a convenient method to understand the impact of a given choice of multi-sample objective function. We call a function $u: [0, 1] \to [0, 1]$ a utility function if it is monotonically increasing. We call a utility function explorative if it is concave. Examples of explorative utility functions have been derived above, and are depicted in Figure~\ref{fig:transform}.

For a single problem $x$, we have
\[
\E_{r_i \sim \pi(x)}[f(r_1, \ldots, r_n)] = u(\E_{r\sim \pi(x)}[r]),
\]
and thus the same holds when taking expectations over $x$. In other words, reinforcement learning with multi-sample objective functions changes the objective to a utility-transformed objective. On a single problem, by monotonicity of $u$, optimal policies for the original objective are optimal for the transformed objective and vice versa. When averaging over a dataset of problems, however, utility functions induce a different allocation of optimization budget. Concretely, let $dx$ be the distribution over problems and $p_x = \E_{r\sim\pi(x)}[r]$. Then
\[
\E_x[p_x] = \int_x p_x dx,
\]
and
\[
\E_x[u(p_x)] = \int_x u(p_x) dx = \int_x p_x \frac{u(p_x)}{p_x} dx,
\]
so $\frac{u(p_x)}{p_x} dx$ can be seen as the importance-reweighted distribution of problem difficulties according to $u$. Explorative utility functions are therefore the utility functions that induce a monotonically decreasing reweighting on problem difficulties, as depicted in Figure~\ref{fig:reweight}.

\section{Variance Reduction with Leave-One-Out Baselines}
\label{app:var_reduction}
Consider as in Appendix~\ref{app:transform} a multi-sample objective function $f$ of zero-one rewards $r_1, \ldots, r_n$ that factors via the mean $p = \frac{1}{n} \sum_i r_i$ via a utility function $u$:
\[
f(r_1, \ldots, r_n) = u(p(r_1, \ldots, r_n)).
\]

Writing $g_i = \nabla_\theta \log \pi_\theta(a_i)$, we are going to compare the default gradient estimator
\[
G_1
= \sum_i f(r_1, \ldots, r_n) g_i
= \sum_i u(p) g_i
\]
with the baselined one
\begin{align*}
G_2
&= \sum_i \left( f(r_1, \ldots, r_n) - f(r_1, \ldots, r_{i-1}, r_{i+1}, \ldots, r_n) \right) g_i \\
&= \sum_i \left(u(p) - u(p_{-i})\right) g_i,
\end{align*}
where
\[
p_{-i} = \frac{1}{n-1} \sum_{j \neq i} r_j.
\]

Because $\E[u(p_{-i}) g_i] = 0$, we have $\E[G_2] = \E[G_1] = \nabla_\theta \E[f(r_1, \ldots, r_n)]$ by the policy gradient theorem.

We will now compare the covariances of $G_1$ and $G_2$, respectively, assuming $u \in C^2(\R_+, \R_+)$ for convenience and setting $\mu = \E[r_1]$.
We have
\begin{equation} \label{eq:p-update}
p = p_{-i} + \frac{1}{n} (r_i - p_{-i})
\end{equation}
and
\[
u(p) = u(p_{-i}) + \frac{1}{n} u'(p_{-i}) (r_i-p_{-i}) + \mathcal{O}\left(\frac{1}{n^2}\right),
\]
hence
\[
\Cov G_2 \approx \Cov \left[ \sum_i \frac{1}{n} u'(p_{-i}) (r_i-p_{-i}) g_i \right]
\approx \Cov \left[ \frac{u'(\mu)}{n}  \sum_i  (r_i-\mu) g_i \right]
= \frac{u'(\mu)^2}{n} \Cov \left[ (r_1-\mu) g_1 \right]
\in \mathcal{O}\left(\frac{1}{n}\right)
\]
by independence of $(r_i-\mu)g_i$ and $(r_j-\mu)g_j$ for $i \neq j$.

For $G_1$, on the other hand, we obtain,
\[
\Cov G_1 = \Cov \left[ \sum_i u(p) g_i \right]
\approx \Cov \left[ \sum_i u(\mu) g_i \right]
= n u(\mu)^2 \Cov \left[ g_1 \right] \in \mathcal{O}(n)
\]

Since $\E[G_1] = \E[G_2] = \nabla_\theta \E[u(\mu)]$ is independent of $n$ in the limit, we obtain signal-to-noise ratios of
\[
\mathrm{SNR}(G_1) = \frac{\|\E[G_1]\|}{\Tr[\Cov G_1]} \in  \mathcal{O}\left(\frac{1}{n}\right)
\]
and
\[
\mathrm{SNR}(G_2) = \frac{\|\E[G_2]\|}{\Tr[\Cov G_2]} \in \mathcal{O}(n),
\]
making $G_1$ inconsistent.
The estimator $G_2$ with the leave-one-out baseline, on the other hand, is asymptotically deterministic.

Note that for simplicity, in the above, we considered a leave-one-out baseline that removes one out of $n$ terms from the aggregation.
In the case of DecompRL, we remove a fraction of $\frac{1}{d}$ out of $md$ terms in the case of decomposition actions and a fraction of $\frac{1}{k}$ terms out of $m$ terms in the case of implementation actions. Since Equation~\ref{eq:p-update} holds with $n$ replaced by $d$ or $k$, resp., and $r_i$ replaced by the average over the removed terms, the same asymptotic analysis holds.

\newpage
\section{Additional Experiments on Multi-Policy Training}
We perform experiments at the 8B scale using the LLama 3.1 8B Instruct model \cite{meta2024llama} to derive our multi-policy training recipe described in section \ref{sec:exp}. We train on the CodeContest training set and evaluate on its validation set. We give additional experiments regarding multi-policy training: benchmark results of combinations of separately trained policies (Figure~\ref{fig:grid}), results on sequential training (Figure~\ref{fig:continue}) and an overall comparison of multi-policy training methods over the course of reinforcement learning runs including alternate training (Figure~\ref{fig:at_10}). See Section~\ref{sec:multi-rl} for additional discussion.

Analyzing the policy entropies for decomposition and implementation in separate training runs, we find that entropy decreases substantially over the course of training, hinting at better certainty about good actions (exploitation) but also harming diversity at large sampling scales (exploration). The conditioning of a fixed implementation policy on an entropy-decreasing decomposition policy leaves the implementation policy entropy unaffected (Figure~\ref{fig:entropy}).

\input{figures/grids}
\input{figures/entropy}

\section{Baselines}
\label{sec:baselines}
We compare DecompRL with four reinforcement learning algorithms:
\begin{itemize}
    \item GRPO \citep{deepseek_grpo} - the state of the art for optimizing pass@$1$ performance.
    \item analytical pass@$k$ \citep{chen2025pass} - the state of the art for optimizing pass@$k$.
    \item Soft Policy Optimization (SPO) \citep{cohen2025soft} - optimized for diverse policies.
    \item our own $\lme$ utility function with regular inference which optimizes the mean between pass@1 and pass@k.
\end{itemize}

We train all baselines until the reward plateaus and tune the hyperparameters based on the best pass@10 performance on LiveCodeBench (see Figure \ref{fig:baselines}). We set $16$ samples per prompt for all baselines, $k = 8$ for pass@$k$ and $\beta = 0.3$ in the logmeanexp. Since DecompRL has a higher sampling budget per problem its reward plateaus after $1$ epoch whereas baselines plateau after $3$ epochs on the same training set.

\begin{figure}[t]
    \centering
    \centering
    \includegraphics[width=\textwidth]{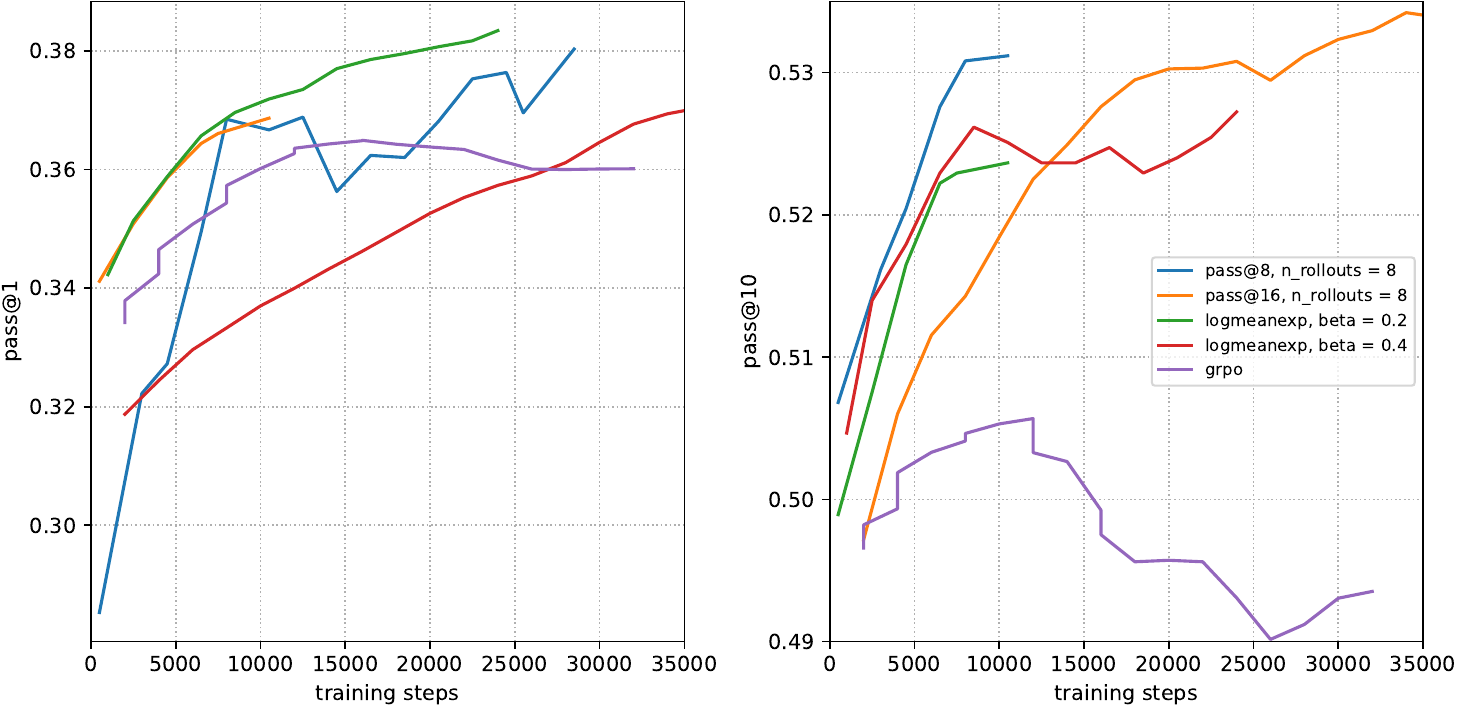}
    \caption{\textbf{Evaluation results during training for baseline methods} on the LiveCodeBench v5 split. Experiments ran with the Code World Model 32B \cite{cwm}. Sampling is performed with nucleus sampling with a probability mass of $0.95$ and temperature $0.6$ for all methods.}
    \label{fig:baselines}
\end{figure}

\section{Hyperparameters and Additional Training Details}
\label{app:hyper}
We train our models with the Adam optimizer \citep{kingma2014adam} ($\beta_1=0.9, \beta_2=0.95$) with decoupled weight decay \citep{loshchilov2019decoupled} factor $0.1$ and gradient clipping at a norm of 1. We warm up the learning rate linearly over 200 steps to a final value of $6 \times 10^{-8}$. We use a local batch size of 2. Trainers discard the first 10 batches at the beginning of a run as warmup.

Our distributed runs use 80 H100 GPUs, split up into 8 trainer GPUs and 72 worker GPUs. On the worker side, we use temperature $1.0$ for generations. During decomposition training, we set the number of implementations $k = 8$, a maximum decomposition size of $n_{\mathrm{max}} = 6$ and conduct up to $m = 1024$ evaluations per decomposition. We run three separate decomposition trainings with $\beta = 0.03, 0.07, 0.1$ with Qwen 2.5 7B which according to Figure~\ref{fig:transform} corresponds to a mixture of pass@k objectives for $1 \leq k \leq 1000$. We pick $\beta = 0.07$ as it had the highest pass@$10$ performance on the DeepMind Code Contest validation set. Each RL training run takes 2--3 days on 80 H100 GPUs for Qwen~2.5~7B. Each evaluation run (up to 1000 decompositions per problem with 4096 evaluations each) takes approximately 20 minutes with tensor parallelism TP=4 on 32 GPUs for Qwen~2.5~7B.

At evaluation time we use temperature of 1.0 and up to 1000 decompositions per problem with each up to $4096$ evaluations.

\paragraph{Memory cost.} DecompRL maintains two copies of the model (decomposition and implementation policies), but only one is trained at a time. For Qwen~2.5~7B, model weights occupy less than 10\% of GPU memory on workers ($\sim$7~GiB for two model copies vs.\ $\sim$3.5~GiB per GPU for a single model). Only worker GPUs bear the extra memory cost. We observed zero out-of-memory errors and required no CPU offloading, even with the Code World Model 32B. We argue the cost of storing two policies is similar to the one of using a reference model to compute a KL penalty in standard RL.

\paragraph{Code execution.} For all our code evaluations, we use a sandbox similar to the one provided for \citet{li2022alphacode} using Python 3.11. We set a time limit of $2T+1$ seconds for all test cases if $T$ is the number of test cases. We execute code on an external CPU cluster that parallelizes unit test executions. Each problem in our training set has a mean of $10$ unit tests, each with average size $<10^2$ bytes which we estimated to take $<0.1$ seconds to execute in our evaluation setup. If we have $512$ evaluations per problem during training, this gives an upper bound of $512 \times 0.1 = 51.2$ seconds per problem.

\section{Prompts}
\label{sec:prompts}
We use language models as decomposition and implementation policies.

For decompositions, we use the following decomposition prompt with a one-shot example, where the \texttt{\{context\}} is the dataset's problem description:
\begin{promptboxpartviolet}[]{Decomposition Prompt}
To solve complicated reasoning problems, it is often helpful to break them down into components.
Can you help a student of computer science solve a competitive programming problem by
breaking it down into sensible component functions?

## Example problem
An inversion in an array is a pair of elements where the first element is greater than the second element,
and the first element appears before the second element in the array.
For example, in the array [2, 4, 1, 3, 5], the inversions are (2, 1), (4, 1), and (4, 3). The total number of inversions is 3.

Task: Write a function that counts the number of inversions in an array efficiently.
The goal is to achieve this in O(n log n) time complexity.

## Example decomposition into component functions
```python
def merge_and_count(left: list[int], right: list[int]) -> tuple[list[int], int]:
    """
    Merges two sorted lists and counts the number of inversions.
    Args:
        left (list[int]): The left sorted subarray.
        right (list[int]): The right sorted subarray.
    Returns:
        tuple[list[int], int]: A tuple containing the merged sorted list and
        the count of inversions found during the merge.
    """
    pass

def merge_sort_and_count(arr: list[int]) -> tuple[list[int], int]:
    """
    Recursively sorts an array and counts the number of inversions.
    Args:
        arr (list[int]): The array to be sorted and analyzed for inversions.
    Returns:
        tuple[list[int], int]: A tuple containing the sorted array and the
        total count of inversions in the array.
    """
    pass

def count_inversions(arr: list[int]) -> int:
    """
    Counts the number of inversions in an array using a modified merge sort.
    Args:
        arr (list[int]): The array to be analyzed for inversions.
    Returns:
        int: The total number of inversions in the array.
    """
    pass
```

In this example decomposition of the programming problem, we give the student the idea
to use count the number of inversions by modifying the merge step of the merge-sort algorithm.
We explain clearly and concisely what the component functions are supposed to do by means of
their type signatures and their docstring, but leave the implementation to the student.

Now, let's do the same for the following competitive programming question:

## Student's competitive programming problem to decompose:
{context}

## Instructions
Please decompose the programming problem into component functions like in the example above.
Guidelines:
- Explain all the functions that the student will need to solve the problem without implementing them.
- Annotate each function with a type signature for its inputs and outputs.
- For each function, write a clear and concise docstring describing what the function does.
- Do not implement the functions: put `pass` as their body.
- The functions can call each other but they should "do one thing" only whereever possible.
    Their behavior should be fully described by their docstring.
- If shared state needs to be passed between the functions, describe it in one of their docstrings
    and make it an argument for the functions that need to read or modify it.
- Enclose the functions in a single code block using triple backticks like so: ```python YOUR CODE HERE ```.
- You can reason about the problem first before starting the code block.
\end{promptboxpartviolet}

After the decomposition step, we extract the function headers and documentation strings from the language model's generation. We implement each function in the decomposition using the prompt below. We have \texttt{\{decomposition\}} extracted from our first language model call, \texttt{\{current\_name\}} the function being implemented, and \texttt{\{other\_names\}} other function descriptions in the decomposition.

\begin{promptboxpartviolet}[]{Implementation Prompt}
I'm trying to solve the following code competitive programming problem:
## Problem description:
{context}

## Instructions
I have decomposed the problem into the following component functions:
```python
{decomposition}
```

Please help me implement the function {current_name}.
You can assume the following functions have been implemented correctly and use them without defining them in your code:
{other_names}

Guidelines:
- Enclose the code in triple backticks like so: ```python YOUR CODE HERE```.
- You can reason about the problem first before starting the code block.
- You can add import statements on top if necessary.

Now, please implement this
```python
{current_code}
```
by filling in the function body, keeping the signature and docstring.
\end{promptboxpartviolet}

\section{Reward Tensor Structure}
\label{app:reward_tensor}

We investigate the structure of reward tensors. For this, we generate solutions to CodeContests validation problems using hierarchical inference with $d=2$ decompositions of a maximal size of $6$ and $k=4$ implementations per function. We evaluate all resulting combinations, i.e. up to $4^6 = 4096$ per problem, and analyze the resulting reward tensors $r$ of mode (number of axes) $n$ and shape $k \times \ldots \times k$. Assume that each implementation is either correct or incorrect, at an individual level, and that combinations pass the tests if and only if all their constituting implementations are correct. In this case, correctness of implementations can be expressed as the maxima $m_i$ across all modes of $r$ except the $i$-th, and the reward tensor would equal the outer product $r^{(1)} = m_1 \otimes \ldots \otimes m_n$ of the $m_i$, which is a rank-1 approximation of $r$. More generally, the rank of a Boolean tensor using the Boolean operations ($\vee$ for addition, and $\wedge$ for multiplication) is known as the Boolean rank \citep{caen1981booleanrank}. In the case of $n=2$ it is also known as the bipartite dimension of the corresponding bipartite graph. It is NP-hard to compute \citep{garey2002computers}.

From 20 decompositions with at least one positive reward, we obtain the following statistics: In 17 cases, $r = r^{(1)}$ agrees with the rank-1 approximation. Overall, viewing $r^{(1)}$ as a prediction of $r$, there are 859 true positives, 483 true negatives, 34 false positives and no false negatives (the latter by definition).

\section{Number of Tokens Generated}
Figure \ref{fig:example}, \ref{fig:dmc_valid}, and Table \ref{tab: lcb_qwen} compare RL-trained models (SPO, GRPO, pass@8, logmeanexp) against our hierarchical inference method (DecompRL) for different token counts. We consider one attempt as one code generation per problem for standard inference and one full decomposition with 8 implementations, 4096 max evaluations for hierarchical inference. To ensure a fair comparison, we compare the pass rates per generated tokens which correspond to a different number of attempts for each RL algorithm. Table \ref{tab:tokens} summarizes the average cost of an attempt per method for the LLama 3.1 8B and Qwen 2.5 7B models.

\input{tables/tokens}

We compare success rates for a fixed budget of tokens similar to the pass@t from \citet{ouyang2022training}.

\section{Example DecompRL}
\label{sec:example_decomp}
We show below an example of a code problem extracted from LiveCodeBench and a generated decomposition with $6$ functions from a Qwen 2.5 7B model trained with DecompRL. This decomposition leads to correct code solutions.

\begin{promptboxpartviolet}[]{Problem Description}
You are given a positive number n.

Return the smallest number x greater than or equal to n, such that the binary representation of x contains only set bits

Example 1:
Input: n = 5 Output: 7
Explanation: The binary representation of 7 is \"111\".

Example 2:
Input: n = 10 Output: 15

Explanation: The binary representation of 15 is \"1111\".

Example 3:
Input: n = 3 Output: 3

Explanation: The binary representation of 3 is \"11\".

Constraints: 1 <= n <= 1000
\end{promptboxpartviolet}

\begin{promptboxpartviolet}{Generated Decomposition}
Certainly! Here's the decomposition of the problem into component functions, along with the type signatures and docstrings for each function:

def next_power_of_two(n: int) -> int:
    pass

def set_bits_until_power_of_two(n: int) -> int:
    pass

def smallest_number(n: int) -> int:
    pass

def count_set_bits(x: int) -> int:
    pass

def main(n: int) -> int:
    pass

def smallestNumber(n: int) -> int:
    power_of_two = next_power_of_two(n)
    set_bits = set_bits_until_power_of_two(power_of_two)
    x = (power_of_two - 1) or set_bits
    return x

In this decomposition:
next_power_of_two finds the smallest power of two greater than or equal to n

set_bits_until_power_of_two counts the number of set bits in n until it reaches the next power of two.

smallest_number uses these two functions to find the smallest number x with all set bits in its binary representation.

count_set_bits simply counts the number of set bits in a given number x.

main is a placeholder for any additional logic.

smallestNumber is the main function that users will call, which leverages the other functions to solve the problem.
\end{promptboxpartviolet}

\section{Hierarchical Inference Without RL}
\label{app:hierarchical_baselines}

\begin{table}[ht]
\centering
\caption{\textbf{Hierarchical inference without RL underperforms standard inference.} Pass@1 comparison at $\sim$4k tokens on LiveCodeBench. Hierarchical (no RL) uses prompted decomposition without any RL training. DecompRL's gains come from RL training, not from the hierarchical structure alone.}
\begin{tabular}{@{}lcc@{}}
\toprule
Model & Method & Pass@1 (\%) \\ \midrule
Qwen 2.5 7B & Standard & 14.0 \\
Qwen 2.5 7B & Hierarchical (no RL) & 11.0 \\
Qwen 2.5 7B & DecompRL & \textbf{18.0} \\ \midrule
CWM 32B & Hierarchical (no RL) & 11.9 \\
CWM 32B & DecompRL & \textbf{27.8} \\
\bottomrule
\label{tab:hierarchical_baselines}
\end{tabular}
\end{table}

\section{Licenses of Used Datasets}
\label{app:licenses}
We use datasets with the following licenses:
\begin{itemize}
    \item CodeContests \citep{li2022alphacode}: CC BY 4.0 with MIT and Apache 2.0 licenses components.
    \item TACO \citep{li2023tacotopicsalgorithmiccode}: Apache 2.0.
    \item LiveCodeBench \citep{jain2024livecodebench}: CC BY 4.0 with MIT and Apache 2.0 licenses components.
\end{itemize}

%% file: proof.tex
\subsection{Variance Reduction}

In this section, we give the proof for the theorem on variance reduction by evaluating recombinations. First, we give an accessible direct proof in the case of two components that exhibits the key idea, then we extend to an arbitrary amount of components.
\begin{theorem}
Let $n \geq 1$, $X_1, \ldots, X_n, Y_1, \ldots, Y_n$ be independent random variables such that all $X_i$ are identically distributed and all $Y_j$ are identically distributed. Let $f(X, Y)$ be a scalar function and assume $\E f(X_1, Y_1)^2 < \infty$. Then
\[
\Var \frac{1}{n^2} \sum_{i, j} f(X_i, Y_j)
\leq \Var \frac{1}{n} \sum_i f(X_i, Y_i)
= \frac{1}{n} \Var f(X_1, Y_1).
\]
\end{theorem}
\begin{proof}
The equality follows from independence of $(X_i, Y_i)$ and $(X_j, Y_j)$ for $i \neq j$. For the proof of the inequality, write $f_X = E[f(X, Y) \mid X]$, $f_Y = E[f(X, Y) \mid Y]$, $\mu = \E[f(X, Y)]$ and consider the Hoeffding decomposition
\[
\Var f(X, Y) = \Var f_X + \Var f_Y + \Var \left( f(X, Y) - f_X - f_Y + \mu \right),
\]
which is proved by expanding all terms and applying the law of total expectation.
In particular, by non-negativity of the variance, we have
\[
\Var f_X + \Var f_Y \leq \Var f(X, Y).
\]

To apply this result, write
\begin{align*}
& \Var \frac{1}{n^2} \sum_{i, j} f(X_i, Y_j) \\
=& \frac{1}{n^4} \sum_{i,j,k,l} \Cov(f(X_i, Y_j), f(X_k, Y_l)) \\
=& \frac{1}{n^4} \sum_{i,j,k,l,i=k \text{ or } j=l} \Cov(f(X_i, Y_j), f(X_k, Y_l)) \\
=& \frac{1}{n^4} \left( \sum_{i,j}\Var f(X_i, Y_j) + \sum_{i,j,l,j \neq l} \Cov(f(X_i, Y_j), f(X_i, Y_l)) + \sum_{i,j,k, k \neq i} \Cov(f(X_i, Y_j), f(X_k, Y_j)) \right) \\
=& \frac{1}{n^2} \Var f(X, Y) + \frac{n-1}{n^2} \Var f_X + \frac{n-1}{n^2} \Var f_Y \\
\leq& \frac{1}{n^2} \Var f(X, Y) + \frac{n-1}{n^2} \Var f(X, Y)  \\
=& \frac{1}{n} \Var f(X, Y).
\end{align*}
Here, we used the fact that $f(X_i, Y_j)$ and $f(X_k, Y_l)$ are independent if $i \neq k$ and $j \neq l$ and made sure not to count the pair of index pairs $(i,j), (i,j)$ twice.
\end{proof}

The generalization to $m$ variables is given in the following theorem. The proof proceeds in the same way but requires the general Hoeffding decomposition involving marginals over all subsets of variables \citep[Chapter 11.4]{Vaart_1998}.
\begin{theorem} \label{thm:var}
Let $n,m \geq 1$, $X^j_i$, $i \leq m$, $j \leq n$ be independent random variables such that $X^j_i, X^k_i$ are identically distributed for all $i,j,k$. Let $f(X_1, \ldots, X_m)$ be a scalar function and assume $\E f(X_1^1, \ldots, X_m^1)^2 < \infty$. Then
\[
\Var \frac{1}{n^m} \sum_{j_i} f(X_1^{j_1}, \ldots, X_m^{j_m})
\leq \Var \frac{1}{n} \sum_j f(X_1^j, \ldots, X_m^j)
= \frac{1}{n} \Var f(X_1^1, \ldots, X_m^1).
\]
\end{theorem}
\begin{proof}
Write $f = f(X_1, \ldots, X_m)$ for short and consider the Hoeffding decomposition defined recursively by
\[
f_\varnothing = \E[f],
\]
and for $A \subseteq \{1, \ldots, m\}$,
\[
f_A = E[f \mid X_i, i \in A] - \sum_{B \subsetneq A} f_B.
\]
This decomposition is orthogonal (i.e. $\E[f_A f_B] = 0$ if $A \neq B$) \citep[Lemma 11.11]{Vaart_1998}, and we have
\[
f = \sum_{S \subseteq \{1, \ldots m\}} f_S
\]
and thus
\[
\Var f = \sum_S \Var f_S.
\]
We now proceed with the proof by writing the variance of the sum into covariance terms and counting contributions like in the case of $m=2$ above. Write $j = (j_1, \ldots, j_m)$ and $k = (k_1, \ldots, k_m)$ for two multi-indices, $j \cap k = \{i \mid j_i = k_i \}$ for the positions where they agree and $f^j$ for $f(X_1^{j_1}, \ldots, X_m^{j_m})$.
We have
\begin{align*}
  &\Var \frac{1}{n^m} \sum_j  f^j \\
= &\frac{1}{n^{2m}} \sum_{j, k} \Cov(f^j, f^k) \\
= &\frac{1}{n^{2m}} \sum_{j, k} \Cov(\sum_S f_S^j, \sum_T f_T^k) \\
= &\frac{1}{n^{2m}} \sum_{j, k} \sum_{S \subseteq j \cap k} \Cov(f_S^j, f_S^k) \\
= &\frac{1}{n^{2m}} \sum_{j, k} \sum_{S \subseteq j \cap k} \Var f_S
\end{align*}
by orthogonality and the fact that $f_S^j$ and $f_S^k$ depend on the same random variables.

What is the coefficient of the term $\Var f_S$ in this sum? Since the indices must agree at the positions in $S$, there are $n$ options for each $j_i = k_i$ for $i \in S$, giving $n^{|S|}$ combinations. At the other indices $i \notin S$, the indices $j_i$ and $k_i$ can take any values ($j_i = k_i$ is possible because $S \subseteq j \cap k$, not $S = j \cap k$), giving another $n^{2m - 2|S|}$ combinations and a total of $n^{2m - |S|}$. Therefore, we get
\[
  \Var \frac{1}{n^m} \sum_j  f^j 
= \frac{1}{n^{2m}} \sum_{j, k} \sum_{S \subseteq j \cap k} \Var f_S 
= \frac{1}{n^{2m}} \sum_S n^{2m - |S|} \Var f_S 
=  \sum_S \frac{1}{n^{|S|}} \Var f_S
\]
and using that $\Var f_\varnothing = \Var \E[f] = 0$ and $|S| \geq 1$ for $S \neq \varnothing$,
\[
= \sum_{S \neq \varnothing} \frac{1}{n^{|S|}} \Var f_S \\
= \sum_{S \neq \varnothing} \frac{1}{n^{|S|}} \Var f_S \\
\leq \sum_{S \neq \varnothing} \frac{1}{n} \Var f_S \\
= \frac{1}{n} \Var f
\]
according to the variance decomposition stated above.
\end{proof}
The bound in the theorem is tight, as evidenced by functions $f$ which are linear in scalar random variables $X_i$. In the other extreme, if $f(X_1, \ldots, X_n) = \prod_i X_i$ and $\E[X_i] = 0$ for all $i$, then all non-diagonal covariance terms disappear and the left-hand side is $\frac{1}{n^{2m-1}} \Var f$ while the right-hand side is $\frac{1}{n} \Var f$.

Let $X_1, \ldots X_n$ and $Y_1, \ldots, Y_n$ be independent random variables such that $X_i$ identically distributed and $Y_i$ identically distributed. Let $f$ be a deterministic function such that $Z_i = f(X_i, Y_i)$ has finite mean $\mu$ and finite variance $\sigma^2$.

According to the central limit theorem, with $S_n = \frac{Z_1 + \ldots + Z_n}{n}$ we have
\[
\sqrt{n} (S_n - \mu) \to \Norm(0, \sigma^2)
\]
in distribution.

Now consider the recombining estimator
\[
T_n = \frac{1}{n^2} \sum_{i,j \leq n} f(X_i, Y_j).
\]

Clearly, $\E T_n = \mu$. In theorem \ref{thm:var} we established
\[
\Var T_n \leq \Var S_n = \frac{1}{n} \sigma^2
\]
using the Hoeffding decomposition of $f(X_i, Y_j)$.

More precisely, setting $f_X = \E[f(X, Y) \mid X] - \mu$, $f_Y = \E[f(X, Y) \mid Y] - \mu$, $f_{XY} = f(X, Y) - f_X - f_Y - \mu$ and $f'(X, Y) = \mu + f_X + f_Y + \frac{1}{\sqrt{n}} f_{XY}$, we derived (implicitly in the proof, and noting the change of notation for $f_X$, $f_Y$)
\[
\Var T_n = \frac{1}{n} \Var f_X + \frac{1}{n} \Var f_Y + \frac{1}{n^2} \Var f_{XY} = \frac{1}{n} \Var f'.
\]
Now set $S_n' = \frac{f'(X_1, Y_1) + \ldots + f'(X_n, Y_n)}{n}$, observe that $\E f' = \mu$ and write
\[
\sigma'^2 = \Var f' = \Var f_X + \Var f_Y + \frac{1}{n} \Var f_{XY} \leq \sigma^2.
\]
According to the central limit theorem,
\[
\sqrt{n} (S'_n - \mu) \to \Norm(0, \sigma'^2)
\]

We will now prove that \footnote{Similar to \url{https://math.stackexchange.com/questions/2760951/}.}
\[
\sqrt{n} \|T_n - S'_n \|_2 \to 0
\]
with $n \to \infty$.
Indeed,
\[
n \| T_n - S'_n \|_2^2 = 
\frac{1}{n^3} \sum_{i,j,k,l} \left(f(X_i, Y_j) - f'(X_i, Y_i)\right)
\left(f(X_k, Y_l) - f'(X_k, Y_k) \right)
\]
Subtracting $\mu$ from all $f$- and $f'$-terms, proceeding as in the proof of Prop. B.1 and observing that $\Cov(f(X_i, Y_j), f'(X_i, Y_i)) = \Var f_X$,  $\Cov(f(X_i, Y_j), f'(X_j, Y_j)) = \Var f_Y$,  $\Cov(f(X_i, Y_i), f'(X_i, Y_i)) = \Var f'$, we get
\begin{align*}
n \| T_n - S'_n \|_2^2
&= \frac{1}{n^3} \left(
n^2 \Var f
+ n^3 \Var f'
- n^2(n-1) \Var f_X
- n^2(n-1) \Var f_Y
\right) \\
&= \frac{1}{n^3} \left(
n^2 \Var f
+ n^3 (\Var f_X + \Var f_Y + \frac{1}{n} \Var f_{XY})
- n^2(n-1) \Var f_X
- n^2(n-1) \Var f_Y
\right),
\end{align*}
so all $n^3$ terms within the brackets cancel and
\[
n \| T_n - S'_n \|_2^2 \to 0
\]
with $n \to \infty$ as claimed.

Since $\sqrt{n} (T_n - \mu) \to \sqrt{n} (S_n' - \mu)$ in $L_2$ and $\sqrt{n} (S_n' - \mu) \to \Norm(0, \sigma'^2)$ in distribution, $\sqrt{n} (T_n - \mu) \to \Norm(0, \sigma'^2)$ in distribution, satisfying a CLT.

We now state the Delta Method theorem\footnote{\url{https://www.stat.rice.edu/~dobelman/notes_papers/math/TaylorAppDeltaMethod.pdf}}.

\begin{theorem}
Let $S_n$ be a sequence of random variables that satisfies $\sqrt{n} (S_n - \mu) \to \Norm(0, \sigma^2)$ in distribution. For a given function $u$, suppose that $u'(\mu)$ exists and is not 0.

Then, 
\[
\sqrt{n} (u(S_n) - u(\mu)) \to \Norm(0, \sigma^2 u'(\mu)^2)
\]
in distribution.
\end{theorem}

We have seen that in the setup of our paper with $u$ a concave utility function implicitly defined by a reward aggregation both the recombination estimator $T_n$ and the standard (diagonal) estimator $S_n$ satisfy a CLT with variances $\sigma'^2$ and $\sigma^2$, respectively. Applying the Delta Method theorem, we can conclude that $u(T_n)$ and $u(S_n)$ likewise satisfy a CLT but with different variances $\sigma'^2 u'(\mu)$ and $\sigma^2 u'(\mu)$, respectively. Since $\sigma'^2 = \sigma^2 - \frac{n-1}{n} \Var f_{XY}$, the recombination estimator is an improvement over the standard estimator, not just for $u = \mathrm{id}$, but also in the presence of nonlinear utility functions (reward aggregation functions) $u$ in the $n \to \infty$ limit.

%% file: figures/reweight.tex
\begin{figure}[t]
\centering
\begin{subfigure}[t]{0.33\textwidth}
\centering
\includegraphics[width=\textwidth]{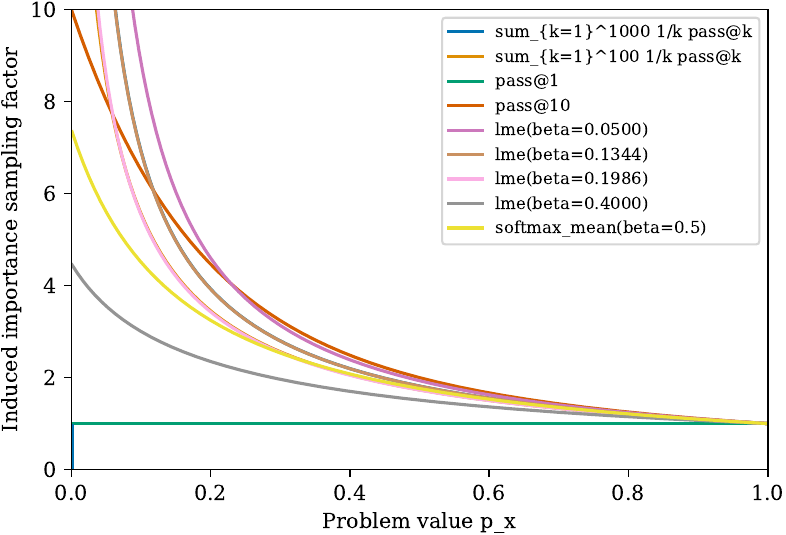}
\caption{Importance sampling weights.}
\label{fig:reweight}
\end{subfigure}
\hfill
\begin{subfigure}[t]{0.3\textwidth}
\centering
\includegraphics[width=\textwidth]{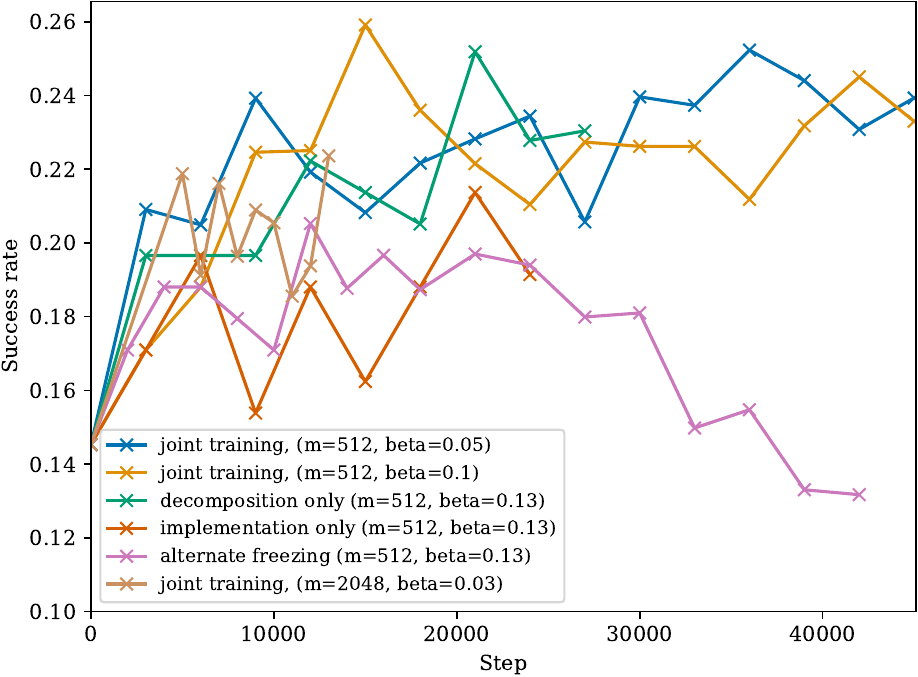}
\caption{Decomposition bottleneck.}
\label{fig:at_10}
\end{subfigure}
\hfill
\begin{subfigure}[t]{0.3\textwidth}
\centering
\includegraphics[width=\textwidth]{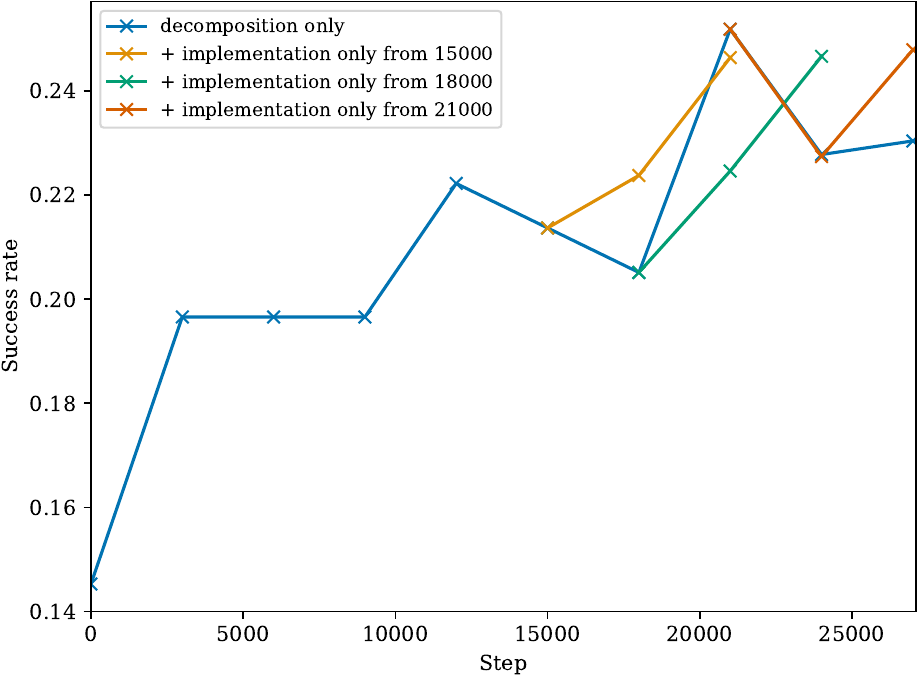}
\caption{Sequential training.}
\label{fig:continue}
\end{subfigure}
\caption{\textbf{Training dynamics and ablations.} (a)~Importance sampling weights induced by explorative utility functions enable jointly optimizing pass@$1$ to pass@$k$. (b)~Ablation on CodeContests validation pass@$10$ shows that removing implementation training does not hurt performance, indicating the decomposition policy is the main RL bottleneck. (c)~Continuing a decomposition-only run with implementation training at 15k, 18k, and 21k steps yields checkpoints among the best but within variance of decomposition-only results.}
\label{fig:training_ablations}
\end{figure}

%% file: figures/grids.tex
\begin{figure}[H]
    \centering
    \begin{subfigure}[b]{0.3\textwidth}
        \centering
        \includegraphics[width=\textwidth]{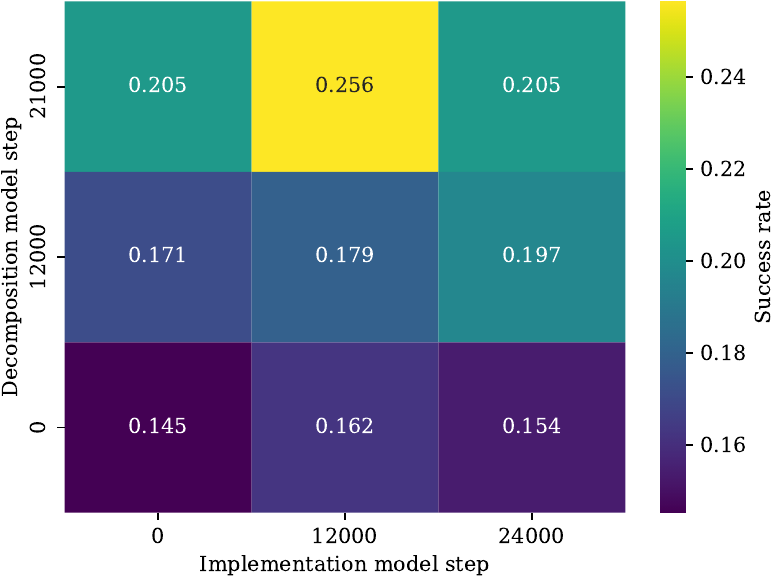}
        \label{fig:grid1}
    \end{subfigure}
    \hfill
    \begin{subfigure}[b]{0.3\textwidth}
        \centering
        \includegraphics[width=\textwidth]{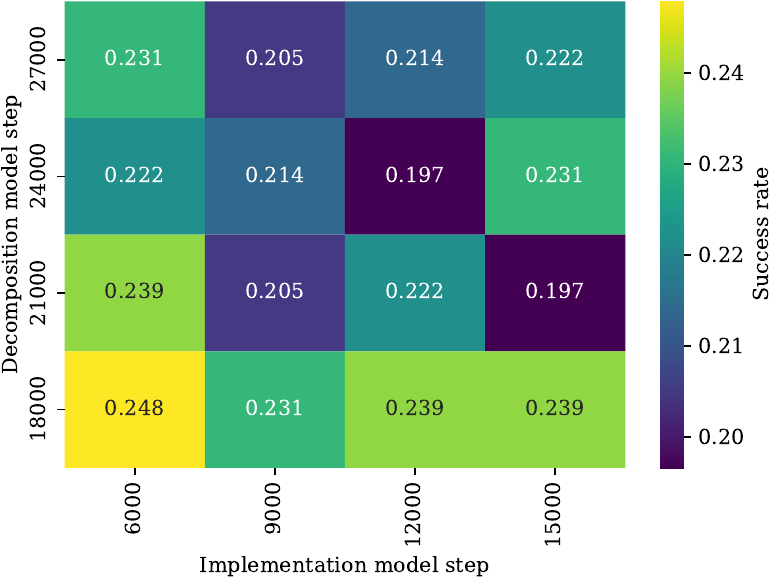}
        \label{fig:grid2}
        
    \end{subfigure}
    \hfill
    \begin{subfigure}[b]{0.3\textwidth}
        \centering
        \includegraphics[width=\textwidth]{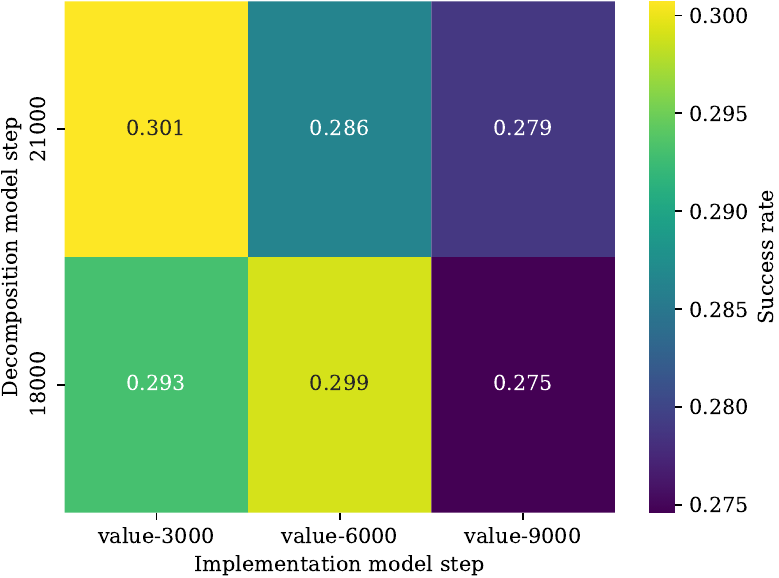}
        \label{fig:grid3}
    \end{subfigure}
    \caption{\textbf{Performance of different combinations of independently trained decomposition and implementation models.} We train a model on decompositions with implementations from fixed Llama 3.1 8B Instruct, and another model on implementations with decompositions from fixed Llama 3.1 8B Instruct, and evaluate cross combinations between the resulting checkpoints. This off-policy estimation is not guaranteed to work in principle but some combinations appear to work in practice. Note that the best configuration is the decomposition checkpoint with 21k steps coupled with the implementation checkpoint with 3k steps, highlighting the importance of joint optimization.
    Evaluations are conducted on CodeContests valid with 4096 function evaluations per decomposition, $k=8$ implementations per function and 10 decompositions per problem.}
    \label{fig:grid}
\end{figure}

%% file: figures/entropy.tex
\begin{figure}[t]
    \centering
    \begin{subfigure}[b]{0.45\textwidth}
        \centering
        \includegraphics[width=\textwidth]{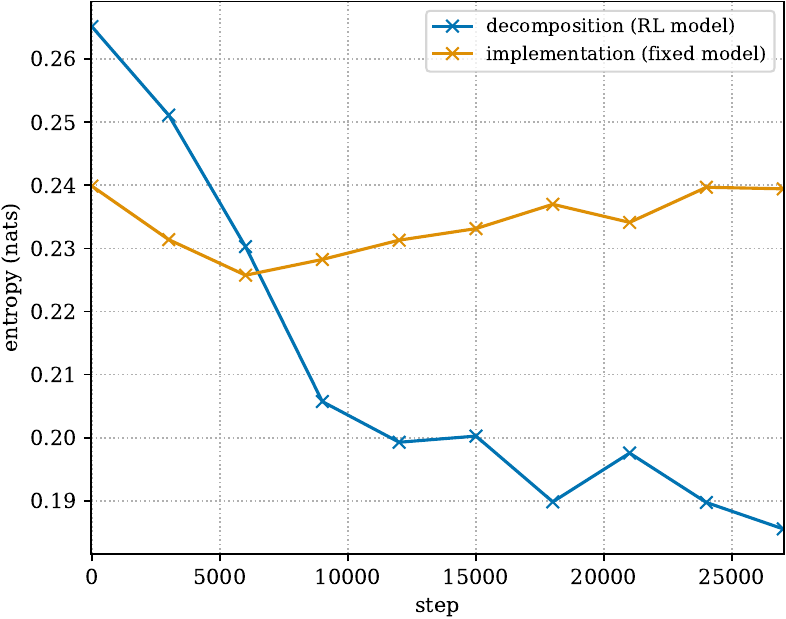}
        \label{fig:entropy_decomp}
    \end{subfigure}
    \hfill
    \begin{subfigure}[b]{0.45\textwidth}
        \centering
        \includegraphics[width=\textwidth]{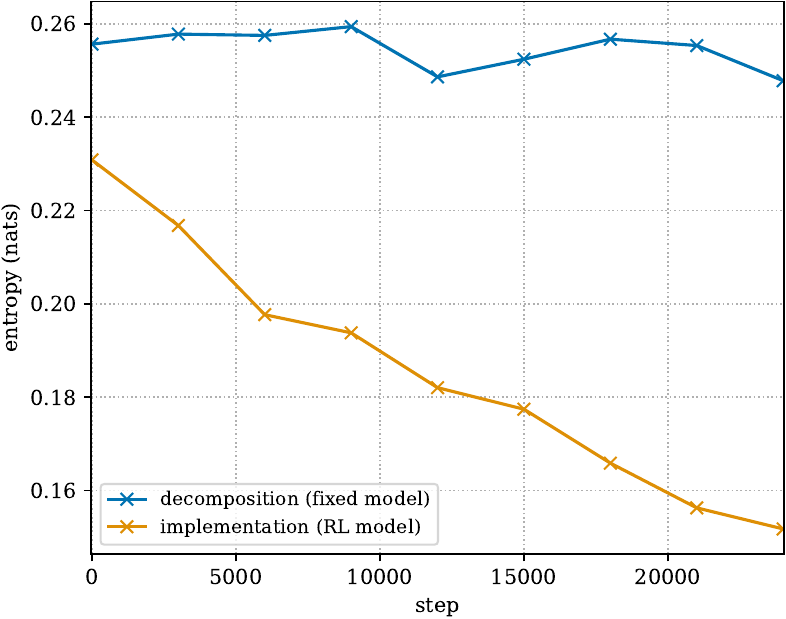}
        \label{fig:entropy_impl}
    \end{subfigure}
    \caption{\textbf{Policy entropy during decomposition policy training (left) and implementation policy training (right).} All ablations in this figure were run during decomposition or implementation policy training, respectively, \emph{not} at evaluation time. \textbf{Left:} decomposition-only training, the implementation policy of a fixed model only changes marginally with the changing prompts from the decomposition model while decomposition policy entropy decreases. \textbf{Right:} implementation-only training; the implementation policy entropy decreases over the course of training.}
    \label{fig:entropy}
\end{figure}

%% file: tables/tokens.tex
\begin{table}[H]
\centering
\begin{tabular}{llr}
\hline
\textbf{Model} & \textbf{RL Algorithm} & \textbf{Avg Tokens per Attempt} \\
\hline
\multirow{5}{*}{Llama 3.1 8B Instruct} & no training & 317 \\
 & GRPO & 215 \\
 & pass@8 & 942 \\
 & SPO & 295 \\
 & DecompRL & 10495 \\
\hline
\multirow{6}{*}{Qwen 2.5 7B} 
& no training & 387 \\
& GRPO & 545 \\
& pass@8 & 570 \\
& logmeanexp, 16 samples & 500 \\
& logmeanexp, 48 samples & 538 \\
& DecompRL & 4000 \\
\hline
\end{tabular}
\caption{\textbf{Token count per code attempt.} Overall RL training increases the the number of generated tokens per attempt. A DecompRL attempt encompasses 1 decomposition and $n \times 8$ implementations for $n$ functions so it contains much more tokens.}
\label{tab:tokens}
\end{table}

%% file: paper.bbl
\begin{thebibliography}{66}
\providecommand{\natexlab}[1]{#1}
\providecommand{\url}[1]{\texttt{#1}}
\expandafter\ifx\csname urlstyle\endcsname\relax
  \providecommand{\doi}[1]{doi: #1}\else
  \providecommand{\doi}{doi: \begingroup \urlstyle{rm}\Url}\fi

\bibitem[AI~@~Meta(2024)]{meta2024llama}
Llama~Team AI~@~Meta.
\newblock The {{Llama}} 3 {{Herd}} of {{Models}}, 2024.

\bibitem[Albrecht et~al.(2024)Albrecht, Christianos, and Sch\"afer]{marl-book}
Stefano~V. Albrecht, Filippos Christianos, and Lukas Sch\"afer.
\newblock \emph{Multi-Agent Reinforcement Learning: Foundations and Modern
  Approaches}.
\newblock MIT Press, 2024.
\newblock \url{https://www.marl-book.com}.

\bibitem[Anthony et~al.(2017)Anthony, Tian, and Barber]{anthony2017thinking}
Thomas Anthony, Zheng Tian, and David Barber.
\newblock Thinking fast and slow with deep learning and tree search.
\newblock \emph{Advances in neural information processing systems}, 30, 2017.

\bibitem[Brown et~al.(2024)Brown, Juravsky, Ehrlich, Clark, Le, R{\'e}, and
  Mirhoseini]{brown2024large}
Bradley Brown, Jordan Juravsky, Ryan Ehrlich, Ronald Clark, Quoc~V Le,
  Christopher R{\'e}, and Azalia Mirhoseini.
\newblock Large language monkeys: Scaling inference compute with repeated
  sampling.
\newblock \emph{arXiv preprint arXiv:2407.21787}, 2024.

\bibitem[Chen et~al.(2024)Chen, Tang, Chu, Chen, Wang, Liu, and
  Qin]{divide_conquer}
Jingchang Chen, Hongxuan Tang, Zheng Chu, Qianglong Chen, Zekun Wang, Ming Liu,
  and Bing Qin.
\newblock Divide-and-conquer meets consensus: Unleashing the power of functions
  in code generation, 2024.
\newblock \url{https://arxiv.org/abs/2405.20092}.

\bibitem[Chen et~al.(2021)Chen, Tworek, Jun, Yuan, Pinto, Kaplan, Edwards,
  Burda, Joseph, Brockman, et~al.]{chen2021evaluating}
Mark Chen, Jerry Tworek, Heewoo Jun, Qiming Yuan, Henrique Ponde De~Oliveira
  Pinto, Jared Kaplan, Harri Edwards, Yuri Burda, Nicholas Joseph, Greg
  Brockman, et~al.
\newblock Evaluating large language models trained on code.
\newblock \emph{arXiv preprint arXiv:2107.03374}, 2021.

\bibitem[Chen et~al.(2025)Chen, Qin, Wu, Ling, Ye, Zhao, and Shi]{chen2025pass}
Zhipeng Chen, Xiaobo Qin, Youbin Wu, Yue Ling, Qinghao Ye, Wayne~Xin Zhao, and
  Guang Shi.
\newblock Pass@ k training for adaptively balancing exploration and
  exploitation of large reasoning models.
\newblock \emph{arXiv preprint arXiv:2508.10751}, 2025.

\bibitem[Christiano et~al.(2023)Christiano, Leike, Brown, Martic, Legg, and
  Amodei]{christiano2023deepreinforcementlearninghuman}
Paul Christiano, Jan Leike, Tom~B. Brown, Miljan Martic, Shane Legg, and Dario
  Amodei.
\newblock Deep reinforcement learning from human preferences, 2023.
\newblock \url{https://arxiv.org/abs/1706.03741}.

\bibitem[Cohen et~al.(2025)Cohen, Zhang, Zheng, Tang, Munos, and
  Synnaeve]{cohen2025soft}
Taco Cohen, David~W Zhang, Kunhao Zheng, Yunhao Tang, Remi Munos, and Gabriel
  Synnaeve.
\newblock Soft policy optimization: Online off-policy rl for sequence models.
\newblock \emph{arXiv preprint arXiv:2503.05453}, 2025.

\bibitem[Cui et~al.(2025)Cui, Zhang, Chen, Yuan, Wang, Zuo, Li, Fan, Chen,
  Chen, Liu, Peng, Bai, Ouyang, Cheng, Zhou, and
  Ding]{cui2025entropymechanismreinforcementlearning}
Ganqu Cui, Yuchen Zhang, Jiacheng Chen, Lifan Yuan, Zhi Wang, Yuxin Zuo,
  Haozhan Li, Yuchen Fan, Huayu Chen, Weize Chen, Zhiyuan Liu, Hao Peng, Lei
  Bai, Wanli Ouyang, Yu~Cheng, Bowen Zhou, and Ning Ding.
\newblock The entropy mechanism of reinforcement learning for reasoning
  language models, 2025.
\newblock \url{https://arxiv.org/abs/2505.22617}.

\bibitem[de~Caen et~al.(1981)de~Caen, Gregory, and
  Pullman]{caen1981booleanrank}
Dominique de~Caen, David~A. Gregory, and Norman~J. Pullman.
\newblock The boolean rank of zero-one matrices, 1981.

\bibitem[DeepSeek-AI(2025)]{deepseek_R1}
DeepSeek-AI.
\newblock Deepseek-r1: Incentivizing reasoning capability in llms via
  reinforcement learning, 2025.
\newblock \url{https://arxiv.org/pdf/2501.12948}.

\bibitem[DeepSeek-AI et~al.(2025)DeepSeek-AI, Guo, Yang, Zhang, Song, Zhang,
  Xu, Zhu, Ma, Wang, Bi, et~al.]{guo2025deepseek}
DeepSeek-AI, Daya Guo, Dejian Yang, Haowei Zhang, Junxiao Song, Ruoyu Zhang,
  Runxin Xu, Qihao Zhu, Shirong Ma, Peiyi Wang, Xiao Bi, et~al.
\newblock Deepseek-r1: Incentivizing reasoning capability in llms via
  reinforcement learning, 2025.
\newblock \url{https://arxiv.org/abs/2501.12948}.

\bibitem[Dong and Ma(2025)]{dong2025stpselfplayllmtheorem}
Kefan Dong and Tengyu Ma.
\newblock Stp: Self-play llm theorem provers with iterative conjecturing and
  proving, 2025.
\newblock \url{https://arxiv.org/abs/2502.00212}.

\bibitem[Ellis et~al.(2023)Ellis, Wong, Nye, Sable-Meyer, Cary, Anaya~Pozo,
  Hewitt, Solar-Lezama, and Tenenbaum]{ellis2023dreamcoder}
Kevin Ellis, Lionel Wong, Maxwell Nye, Mathias Sable-Meyer, Luc Cary, Lore
  Anaya~Pozo, Luke Hewitt, Armando Solar-Lezama, and Joshua~B Tenenbaum.
\newblock Dreamcoder: growing generalizable, interpretable knowledge with
  wake--sleep bayesian program learning.
\newblock \emph{Philosophical Transactions of the Royal Society A},
  381\penalty0 (2251):\penalty0 20220050, 2023.

\bibitem[{FAIR CodeGen team} et~al.(2025){FAIR CodeGen team}, :, Copet,
  Carbonneaux, Cohen, Gehring, Kahn, Kossen, Kreuk, McMilin, Meyer, Wei, Zhang,
  Zheng, Armengol-Estapé, Bashiri, Beck, Chambon, Charnalia, Cummins, Decugis,
  Fisches, Fleuret, Gloeckle, Gu, Hassid, Haziza, Idrissi, Keller, Kindi,
  Leather, Maimon, Markosyan, Massa, Mazaré, Mella, Murray, Muzumdar, O'Hearn,
  Pagliardini, Pedchenko, Remez, Seeker, Selvi, Sultan, Wang, Wehrstedt, Yoran,
  Zhang, Cohen, Adi, and Synnaeve]{cwm}
{FAIR CodeGen team}, :, Jade Copet, Quentin Carbonneaux, Gal Cohen, Jonas
  Gehring, Jacob Kahn, Jannik Kossen, Felix Kreuk, Emily McMilin, Michel Meyer,
  Yuxiang Wei, David Zhang, Kunhao Zheng, Jordi Armengol-Estapé, Pedram
  Bashiri, Maximilian Beck, Pierre Chambon, Abhishek Charnalia, Chris Cummins,
  Juliette Decugis, Zacharias~V. Fisches, François Fleuret, Fabian Gloeckle,
  Alex Gu, Michael Hassid, Daniel Haziza, Badr~Youbi Idrissi, Christian Keller,
  Rahul Kindi, Hugh Leather, Gallil Maimon, Aram Markosyan, Francisco Massa,
  Pierre-Emmanuel Mazaré, Vegard Mella, Naila Murray, Keyur Muzumdar, Peter
  O'Hearn, Matteo Pagliardini, Dmitrii Pedchenko, Tal Remez, Volker Seeker,
  Marco Selvi, Oren Sultan, Sida Wang, Luca Wehrstedt, Ori Yoran, Lingming
  Zhang, Taco Cohen, Yossi Adi, and Gabriel Synnaeve.
\newblock Cwm: An open-weights llm for research on code generation with world
  models, 2025.
\newblock \url{https://arxiv.org/abs/2510.02387}.

\bibitem[Feng et~al.(2024)Feng, Wan, Wen, McAleer, Wen, Zhang, and
  Wang]{tree_search}
Xidong Feng, Ziyu Wan, Muning Wen, Stephen~Marcus McAleer, Ying Wen, Weinan
  Zhang, and Jun Wang.
\newblock Alphazero-like tree-search can guide large language model decoding
  and training, 2024.
\newblock \url{https://arxiv.org/abs/2309.17179}.

\bibitem[Foerster et~al.(2024)Foerster, Farquhar, Afouras, Nardelli, and
  Whiteson]{foerster2024counterfactualmultiagentpolicygradients}
Jakob Foerster, Gregory Farquhar, Triantafyllos Afouras, Nantas Nardelli, and
  Shimon Whiteson.
\newblock Counterfactual multi-agent policy gradients, 2024.
\newblock \url{https://arxiv.org/abs/1705.08926}.

\bibitem[Garey and Johnson(2002)]{garey2002computers}
Michael~R Garey and David~S Johnson.
\newblock \emph{Computers and intractability}, volume~29.
\newblock wh freeman New York, 2002.

\bibitem[Gauthier et~al.(2023)Gauthier, Ol{\v{s}}{\'a}k, and
  Urban]{gauthier2023alien}
Thibault Gauthier, Miroslav Ol{\v{s}}{\'a}k, and Josef Urban.
\newblock Alien coding.
\newblock \emph{International Journal of Approximate Reasoning}, 162:\penalty0
  109009, 2023.

\bibitem[Gehring et~al.(2025)Gehring, Zheng, Copet, Mella, Carbonneaux, Cohen,
  and Synnaeve]{RLEF}
Jonas Gehring, Kunhao Zheng, Jade Copet, Vegard Mella, Quentin Carbonneaux,
  Taco Cohen, and Gabriel Synnaeve.
\newblock Rlef: Grounding code llms in execution feedback with reinforcement
  learning, 2025.
\newblock \url{https://arxiv.org/abs/2410.02089}.

\bibitem[Grayeli et~al.(2024)Grayeli, Sehgal, Costilla~Reyes, Cranmer, and
  Chaudhuri]{grayeli2024symbolic}
Arya Grayeli, Atharva Sehgal, Omar Costilla~Reyes, Miles Cranmer, and Swarat
  Chaudhuri.
\newblock Symbolic regression with a learned concept library.
\newblock \emph{Advances in Neural Information Processing Systems},
  37:\penalty0 44678--44709, 2024.

\bibitem[Gulcehre et~al.(2023)Gulcehre, Paine, Srinivasan, Konyushkova, Weerts,
  Sharma, Siddhant, Ahern, Wang, Gu, et~al.]{rest}
Caglar Gulcehre, Tom~Le Paine, Srivatsan Srinivasan, Ksenia Konyushkova, Lotte
  Weerts, Abhishek Sharma, Aditya Siddhant, Alex Ahern, Miaosen Wang, Chenjie
  Gu, et~al.
\newblock Reinforced self-training (rest) for language modeling.
\newblock \emph{arXiv preprint arXiv:2308.08998}, 2023.

\bibitem[Haller et~al.(2024)Haller, Ringer, Rute, and Miranda]{haller2024peano}
Gabriel Haller, Talia Ringer, Jason Rute, and Brando Miranda.
\newblock Peano: Learning formal mathematical reasoning, 2024.
\newblock \url{https://arxiv.org/abs/2405.06738}.

\bibitem[Haluptzok et~al.(2022)Haluptzok, Bowers, and
  Kalai]{haluptzok2022language}
Patrick Haluptzok, Matthew Bowers, and Adam~Tauman Kalai.
\newblock Language models can teach themselves to program better.
\newblock \emph{arXiv preprint arXiv:2207.14502}, 2022.

\bibitem[Ho et~al.(2024)Ho, Shen, Pan, Fang, Liu, Liu, and
  Zhang]{liu2024openllm}
Chia-Tung Ho, Yikang Shen, Jingyu Pan, Chao Fang, Hao Liu, Tianyu Liu, and
  Zhiru Zhang.
\newblock Openllm-rtl: Open dataset and benchmark for llm-aided design of
  digital circuits, 2024.
\newblock \url{https://arxiv.org/abs/2407.14326}.

\bibitem[Hughes et~al.(2024)Hughes, Price, Lynch, Schaeffer, Barez, Koyejo,
  Sleight, Jones, Perez, and Sharma]{hughes2024bestofnjailbreaking}
John Hughes, Sara Price, Aengus Lynch, Rylan Schaeffer, Fazl Barez, Sanmi
  Koyejo, Henry Sleight, Erik Jones, Ethan Perez, and Mrinank Sharma.
\newblock Best-of-n jailbreaking, 2024.
\newblock \url{https://arxiv.org/abs/2412.03556}.

\bibitem[Islam et~al.(2024)Islam, Ali, and Parvez]{mapcoder}
Md.~Ashraful Islam, Mohammed~Eunus Ali, and Md~Rizwan Parvez.
\newblock Mapcoder: Multi-agent code generation for competitive problem
  solving, 2024.
\newblock \url{https://arxiv.org/abs/2405.11403}.

\bibitem[Jain et~al.(2024{\natexlab{a}})Jain, Han, Gu, Li, Yan, Zhang, Wang,
  Solar-Lezama, Sen, and Stoica]{jain2024livecodebench}
Naman Jain, King Han, Alex Gu, Wen-Ding Li, Fanjia Yan, Tianjun Zhang, Sida
  Wang, Armando Solar-Lezama, Koushik Sen, and Ion Stoica.
\newblock Livecodebench: Holistic and contamination free evaluation of large
  language models for code.
\newblock \emph{arXiv preprint arXiv:2403.07974}, 2024{\natexlab{a}}.

\bibitem[Jain et~al.(2024{\natexlab{b}})Jain, Zhang, Chiang, Gonzalez, Sen, and
  Stoica]{jain2023llm}
Naman Jain, Tianjun Zhang, Wei{-}Lin Chiang, Joseph~E. Gonzalez, Koushik Sen,
  and Ion Stoica.
\newblock Llm-assisted code cleaning for training accurate code generators.
\newblock In \emph{The Twelfth International Conference on Learning
  Representations, {ICLR} 2024, Vienna, Austria, May 7-11, 2024}.
  OpenReview.net, 2024{\natexlab{b}}.
\newblock \url{https://openreview.net/forum?id=maRYffiUpI}.

\bibitem[Khot et~al.(2023)Khot, Trivedi, Finlayson, Fu, Richardson, Clark, and
  Sabharwal]{khot2022decomposed}
Tushar Khot, Harsh Trivedi, Matthew Finlayson, Yao Fu, Kyle Richardson, Peter
  Clark, and Ashish Sabharwal.
\newblock Decomposed prompting: {A} modular approach for solving complex tasks.
\newblock In \emph{The Eleventh International Conference on Learning
  Representations, {ICLR} 2023, Kigali, Rwanda, May 1-5, 2023}. OpenReview.net,
  2023.
\newblock \url{https://openreview.net/forum?id=\_nGgzQjzaRy}.

\bibitem[Kingma and Ba(2015)]{kingma2014adam}
Diederik Kingma and Jimmy Ba.
\newblock Adam: A method for stochastic optimization.
\newblock \emph{ICLR}, 2015.

\bibitem[Lample et~al.(2022)Lample, Lachaux, Lavril, Martinet, Hayat, Ebner,
  Rodriguez, and Lacroix]{lample2022hypertree}
Guillaume Lample, Marie-Anne Lachaux, Thibaut Lavril, Xavier Martinet, Amaury
  Hayat, Gabriel Ebner, Aurélien Rodriguez, and Timothée Lacroix.
\newblock Hypertree proof search for neural theorem proving.
\newblock \emph{arXiv preprint arXiv:2205.11491}, 2022.
\newblock \url{https://doi.org/10.48550/arXiv.2205.11491}.

\bibitem[Le et~al.(2023)Le, Chen, Saha, Gokul, Sahoo, and
  Joty]{le2023codechain}
Hung Le, Hailin Chen, Amrita Saha, Akash Gokul, Doyen Sahoo, and Shafiq Joty.
\newblock Codechain: Towards modular code generation through chain of
  self-revisions with representative sub-modules.
\newblock \emph{arXiv preprint arXiv:2310.08992}, 2023.

\bibitem[Li et~al.(2023)Li, Fu, Zhang, Huang, Sun, Lyu, Liu, Jin, and
  Li]{li2023tacotopicsalgorithmiccode}
Rongao Li, Jie Fu, Bo-Wen Zhang, Tao Huang, Zhihong Sun, Chen Lyu, Guang Liu,
  Zhi Jin, and Ge~Li.
\newblock Taco: Topics in algorithmic code generation dataset, 2023.
\newblock \url{https://arxiv.org/abs/2312.14852}.

\bibitem[Li et~al.(2022)Li, Choi, Chung, Kushman, Schrittwieser, Leblond,
  Eccles, Keeling, Gimeno, Dal~Lago, et~al.]{li2022alphacode}
Yujia Li, David Choi, Junyoung Chung, Nate Kushman, Julian Schrittwieser,
  R{\'e}mi Leblond, Tom Eccles, James Keeling, Felix Gimeno, Agustin Dal~Lago,
  et~al.
\newblock Competition-level code generation with alphacode.
\newblock \emph{Science}, 378\penalty0 (6624):\penalty0 1092--1097, 2022.

\bibitem[Light et~al.(2025)Light, Wu, Sun, Yu, Liu, Zhao, Hu, Chen, and
  Cheng]{light2025sfs}
Jonathan Light, Yue Wu, Yiyou Sun, Wenchao Yu, Yanchi Liu, Xujiang Zhao, Ziniu
  Hu, Haifeng Chen, and Wei Cheng.
\newblock {SFS}: Smarter code space search improves {LLM} inference scaling.
\newblock In \emph{The Thirteenth International Conference on Learning
  Representations}, 2025.
\newblock \url{https://openreview.net/forum?id=MCHuGOkExF}.

\bibitem[Lin et~al.(2025)Lin, Tang, Lyu, Wu, Lin, Yang, Li, Xia, Chen, Arora,
  and Jin]{lin2025goedelproverfrontiermodelopensource}
Yong Lin, Shange Tang, Bohan Lyu, Jiayun Wu, Hongzhou Lin, Kaiyu Yang, Jia Li,
  Mengzhou Xia, Danqi Chen, Sanjeev Arora, and Chi Jin.
\newblock Goedel-prover: A frontier model for open-source automated theorem
  proving, 2025.
\newblock \url{https://arxiv.org/abs/2502.07640}.

\bibitem[Loshchilov and Hutter(2019)]{loshchilov2019decoupled}
Ilya Loshchilov and Frank Hutter.
\newblock Decoupled weight decay regularization, 2019.

\bibitem[Madaan et~al.(2023)Madaan, Tandon, Gupta, Hallinan, Gao, Wiegreffe,
  Alon, Dziri, Prabhumoye, Yang, et~al.]{madaan2023self}
Aman Madaan, Niket Tandon, Prakhar Gupta, Skyler Hallinan, Luyu Gao, Sarah
  Wiegreffe, Uri Alon, Nouha Dziri, Shrimai Prabhumoye, Yiming Yang, et~al.
\newblock Self-refine: Iterative refinement with self-feedback.
\newblock \emph{Advances in Neural Information Processing Systems},
  36:\penalty0 46534--46594, 2023.

\bibitem[Mnih et~al.(2013)Mnih, Kavukcuoglu, Silver, Graves, Antonoglou,
  Wierstra, and Riedmiller]{mnih2013playing}
Volodymyr Mnih, Koray Kavukcuoglu, David Silver, Alex Graves, Ioannis
  Antonoglou, Daan Wierstra, and Martin Riedmiller.
\newblock Playing atari with deep reinforcement learning.
\newblock \emph{arXiv preprint arXiv:1312.5602}, 2013.

\bibitem[Noukhovitch et~al.(2025)Noukhovitch, Huang, Xhonneux, Hosseini,
  Agarwal, and Courville]{rlhf_faster}
Michael Noukhovitch, Shengyi Huang, Sophie Xhonneux, Arian Hosseini, Rishabh
  Agarwal, and Aaron Courville.
\newblock Asynchronous rlhf: Faster and more efficient off-policy rl for
  language models, 2025.
\newblock \url{https://arxiv.org/abs/2410.18252}.

\bibitem[OpenAI(2023)]{openai2023gpt4}
OpenAI.
\newblock Gpt-4 technical report, 2023.

\bibitem[Ouyang et~al.(2022)Ouyang, Wu, Jiang, Almeida, Wainwright, Mishkin,
  Zhang, Agarwal, Slama, Ray, et~al.]{ouyang2022training}
Long Ouyang, Jeffrey Wu, Xu~Jiang, Diogo Almeida, Carroll Wainwright, Pamela
  Mishkin, Chong Zhang, Sandhini Agarwal, Katarina Slama, Alex Ray, et~al.
\newblock Training language models to follow instructions with human feedback.
\newblock \emph{Advances in neural information processing systems},
  35:\penalty0 27730--27744, 2022.

\bibitem[Poesia et~al.(2024)Poesia, Broman, Haber, and
  Goodman]{poesia2024learning}
Gabriel Poesia, David Broman, Nick Haber, and Noah Goodman.
\newblock Learning formal mathematics from intrinsic motivation.
\newblock \emph{Advances in Neural Information Processing Systems},
  37:\penalty0 43032--43057, 2024.

\bibitem[Polu et~al.(2022)Polu, Han, Zheng, Baksys, Babuschkin, and
  Sutskever]{polu2022formal}
Stanislas Polu, Jesse~Michael Han, Kunhao Zheng, Mantas Baksys, Igor
  Babuschkin, and Ilya Sutskever.
\newblock Formal mathematics statement curriculum learning, 2022.

\bibitem[Qwen et~al.(2025)Qwen, :, Yang, Yang, Zhang, Hui, Zheng, Yu, Li, Liu,
  Huang, Wei, Lin, Yang, Tu, Zhang, Yang, Yang, Zhou, Lin, Dang, Lu, Bao, Yang,
  Yu, Li, Xue, Zhang, Zhu, Men, Lin, Li, Tang, Xia, Ren, Ren, Fan, Su, Zhang,
  Wan, Liu, Cui, Zhang, and Qiu]{qwen2025qwen25technicalreport}
Qwen, :, An~Yang, Baosong Yang, Beichen Zhang, Binyuan Hui, Bo~Zheng, Bowen Yu,
  Chengyuan Li, Dayiheng Liu, Fei Huang, Haoran Wei, Huan Lin, Jian Yang,
  Jianhong Tu, Jianwei Zhang, Jianxin Yang, Jiaxi Yang, Jingren Zhou, Junyang
  Lin, Kai Dang, Keming Lu, Keqin Bao, Kexin Yang, Le~Yu, Mei Li, Mingfeng Xue,
  Pei Zhang, Qin Zhu, Rui Men, Runji Lin, Tianhao Li, Tianyi Tang, Tingyu Xia,
  Xingzhang Ren, Xuancheng Ren, Yang Fan, Yang Su, Yichang Zhang, Yu~Wan,
  Yuqiong Liu, Zeyu Cui, Zhenru Zhang, and Zihan Qiu.
\newblock Qwen2.5 technical report, 2025.
\newblock \url{https://arxiv.org/abs/2412.15115}.

\bibitem[Rastogi et~al.(2025)Rastogi, Shao, Padhye, and Yang]{rastogi2025learn}
Srishti Rastogi, Yijia Shao, Rohan Padhye, and Diyi Yang.
\newblock Learn to reason efficiently with adaptive length-based reward
  shaping, 2025.
\newblock \url{https://arxiv.org/abs/2504.01191}.

\bibitem[Rosipal and Girolami(2001)]{rosipal01}
R.~Rosipal and M.~Girolami.
\newblock An expectation-maximization approach to nonlinear component analysis.
\newblock \emph{Neural Computation}, 13:\penalty0 505--510, 2001.

\bibitem[Schulman et~al.(2017)Schulman, Wolski, Dhariwal, Radford, and
  Klimov]{ppo}
John Schulman, Filip Wolski, Prafulla Dhariwal, Alec Radford, and Oleg Klimov.
\newblock Proximal policy optimization algorithms, 2017.
\newblock \url{https://arxiv.org/abs/1707.06347}.

\bibitem[Shao et~al.(2024)Shao, Wang, Zhu, Xu, Song, Bi, Zhang, Zhang, Li, Wu,
  and Guo]{deepseek_grpo}
Zhihong Shao, Peiyi Wang, Qihao Zhu, Runxin Xu, Junxiao Song, Xiao Bi, Haowei
  Zhang, Mingchuan Zhang, Y.~K. Li, Y.~Wu, and Daya Guo.
\newblock Deepseekmath: Pushing the limits of mathematical reasoning in open
  language models, 2024.
\newblock \url{https://arxiv.org/abs/2402.03300}.

\bibitem[Shi et~al.(2024)Shi, Wang, Wan, and Gu]{shi2024code}
Yuling Shi, Songsong Wang, Chengcheng Wan, and Xiaodong Gu.
\newblock From code to correctness: Closing the last mile of code generation
  with hierarchical debugging.
\newblock \emph{arXiv preprint arXiv:2410.01215}, 2024.

\bibitem[Shinn et~al.(2023)Shinn, Cassano, Gopinath, Narasimhan, and
  Yao]{shinn2023reflexion}
Noah Shinn, Federico Cassano, Ashwin Gopinath, Karthik Narasimhan, and Shunyu
  Yao.
\newblock Reflexion: Language agents with verbal reinforcement learning.
\newblock \emph{Advances in Neural Information Processing Systems},
  36:\penalty0 8634--8652, 2023.

\bibitem[Sutton and Barto(1998)]{sutton98book}
R.~Sutton and A.~Barto.
\newblock \emph{Reinforcement learning: An introduction}.
\newblock MIT Press, 1998.

\bibitem[Tang et~al.(2025)Tang, Zheng, Synnaeve, and Munos]{tang2025optimizing}
Yunhao Tang, Kunhao Zheng, Gabriel Synnaeve, and R{\'e}mi Munos.
\newblock Optimizing language models for inference time objectives using
  reinforcement learning.
\newblock \emph{arXiv preprint arXiv:2503.19595}, 2025.

\bibitem[Teodorescu et~al.(2023)Teodorescu, Colas, Bowers, Carta, and
  Oudeyer]{teodorescu2023codeplay}
Laetitia Teodorescu, C{\'e}dric Colas, Matthew Bowers, Thomas Carta, and
  Pierre-Yves Oudeyer.
\newblock Codeplay: Autotelic learning through collaborative self-play in
  programming environments.
\newblock In \emph{IMOL 2023-Intrinsically Motivated Open-ended Learning
  workshop at NeurIPS 2023}, 2023.

\bibitem[Vaart(1998)]{Vaart_1998}
A.~W. van~der Vaart.
\newblock \emph{Asymptotic Statistics}.
\newblock Cambridge Series in Statistical and Probabilistic Mathematics.
  Cambridge University Press, 1998.

\bibitem[Wang et~al.(2025)Wang, Cassano, Wu, Bai, Song, Nath, Han, Hendryx,
  Yue, and Zhang]{plansearch}
Evan~Z Wang, Federico Cassano, Catherine Wu, Yunfeng Bai, William Song, Vaskar
  Nath, Ziwen Han, Sean~M. Hendryx, Summer Yue, and Hugh Zhang.
\newblock Planning in natural language improves {LLM} search for code
  generation.
\newblock In \emph{The Thirteenth International Conference on Learning
  Representations}, 2025.
\newblock \url{https://openreview.net/forum?id=48WAZhwHHw}.

\bibitem[Williams(1992)]{10.1007/BF00992696}
Ronald~J. Williams.
\newblock Simple statistical gradient-following algorithms for connectionist
  reinforcement learning.
\newblock \emph{Mach. Learn.}, 8\penalty0 (3–4):\penalty0 229–256, May
  1992.
\newblock ISSN 0885-6125.
\newblock \doi{10.1007/BF00992696}.
\newblock \url{https://doi.org/10.1007/BF00992696}.

\bibitem[Yao et~al.(2023)Yao, Yu, Zhao, Shafran, Griffiths, Cao, and
  Narasimhan]{ToT}
Shunyu Yao, Dian Yu, Jeffrey Zhao, Izhak Shafran, Tom Griffiths, Yuan Cao, and
  Karthik Narasimhan.
\newblock Tree of thoughts: Deliberate problem solving with large language
  models.
\newblock In Alice Oh, Tristan Naumann, Amir Globerson, Kate Saenko, Moritz
  Hardt, and Sergey Levine, editors, \emph{Advances in Neural Information
  Processing Systems 36: Annual Conference on Neural Information Processing
  Systems 2023, NeurIPS 2023, New Orleans, LA, USA, December 10 - 16, 2023},
  2023.
\newblock
  \url{http://papers.nips.cc/paper\_files/paper/2023/hash/271db9922b8d1f4dd7aaef84ed5ac703-Abstract-Conference.html}.

\bibitem[Zelikman et~al.(2022)Zelikman, Mu, Goodman, and Wu]{Star}
Eric Zelikman, Jesse Mu, Noah~D. Goodman, and Yuhuai~Tony Wu.
\newblock Star: Self-taught reasoner bootstrapping reasoning with reasoning.
\newblock 2022.

\bibitem[Zelikman et~al.(2023)Zelikman, Huang, Poesia, Goodman, and
  Haber]{zelikman2023parsel}
Eric Zelikman, Qian Huang, Gabriel Poesia, Noah~D Goodman, and Nick Haber.
\newblock Parsel: A (de-) compositional framework for algorithmic reasoning
  with language models.
\newblock \emph{arXiv preprint arXiv:2212.10561}, 2023.

\bibitem[Zelikman et~al.(2024)Zelikman, Harik, Shao, Jayasiri, Haber, and
  Goodman]{zelikman2024quiet}
Eric Zelikman, Georges Harik, Yijia Shao, Varuna Jayasiri, Nick Haber, and
  Noah~D Goodman.
\newblock Quiet-star: Language models can teach themselves to think before
  speaking.
\newblock \emph{arXiv preprint arXiv:2403.09629}, 2024.

\bibitem[Zhang et~al.(2024)Zhang, Zhoubian, Hu, Yue, Dong, and
  Tang]{zhang2024rest}
Dan Zhang, Sining Zhoubian, Ziniu Hu, Yisong Yue, Yuxiao Dong, and Jie Tang.
\newblock Rest-mcts*: Llm self-training via process reward guided tree search.
\newblock \emph{Advances in Neural Information Processing Systems},
  37:\penalty0 64735--64772, 2024.

\bibitem[Zhou et~al.(2023)Zhou, Sch{\"{a}}rli, Hou, Wei, Scales, Wang,
  Schuurmans, Cui, Bousquet, Le, and Chi]{zhou2022least}
Denny Zhou, Nathanael Sch{\"{a}}rli, Le~Hou, Jason Wei, Nathan Scales, Xuezhi
  Wang, Dale Schuurmans, Claire Cui, Olivier Bousquet, Quoc~V. Le, and Ed~H.
  Chi.
\newblock Least-to-most prompting enables complex reasoning in large language
  models.
\newblock In \emph{The Eleventh International Conference on Learning
  Representations, {ICLR} 2023, Kigali, Rwanda, May 1-5, 2023}. OpenReview.net,
  2023.
\newblock \url{https://openreview.net/forum?id=WZH7099tgfM}.

\bibitem[Zhou et~al.(2024)Zhou, Pujara, Ren, Chen, Cheng, Le, Chi, Zhou,
  Mishra, and Zheng]{self_discover}
Pei Zhou, Jay Pujara, Xiang Ren, Xinyun Chen, Heng-Tze Cheng, Quoc~V. Le, Ed~H.
  Chi, Denny Zhou, Swaroop Mishra, and Huaixiu~Steven Zheng.
\newblock Self-discover: Large language models self-compose reasoning
  structures, 2024.
\newblock \url{https://arxiv.org/abs/2402.03620}.

\end{thebibliography}
